\documentclass[11pt,letterpaper]{article}

\usepackage[utf8]{inputenc} 
\usepackage[T1]{fontenc}    
\usepackage[margin=1in]{geometry} 
\usepackage{hyperref}       
\usepackage{url}            
\usepackage{xparse}         
\usepackage{booktabs}       
\usepackage{amsfonts}       
\usepackage{nicefrac}       
\usepackage{xcolor}         
\usepackage{setspace}       
\usepackage{amsmath}        
\usepackage{cleveref}       
\usepackage{graphicx}       
\usepackage{todonotes}      
\usepackage{amsmath}
\usepackage{amssymb}
\usepackage{mathtools}
\usepackage{amsthm}
\usepackage{algorithm}
\usepackage{algorithmic}
\usepackage{comment}
\theoremstyle{plain}
\newtheorem{theorem}{Theorem}
\newtheorem{lemma}[theorem]{Lemma}
\newtheorem{corollary}[theorem]{Corollary}
\newtheorem{proposition}[theorem]{Proposition}
\newenvironment{reptheorem}[2][]{%
  \begingroup
  \renewcommand{\thetheorem}{\ref{#2}}%
  \addtocounter{theorem}{-1}%
  \begin{theorem}[#1]%
}{%
  \end{theorem}%
  \endgroup
}
\theoremstyle{definition}
\newtheorem{definition}[theorem]{Definition}

\theoremstyle{remark}
\newtheorem{remark}[theorem]{Remark}

\newcommand{\E}{\mathbb{E}}
\renewcommand{\P}{\mathbb{P}}
\newcommand{\ind}{\mathbb{I}}
\newcommand{\ChiSq}{\chi^2}

\newcommand{\KL}{\mathrm{KL}}

\newcommand{\Reg}{\mathrm{Reg}}

\newcommand{\Ber}{\mathrm{Ber}}

\NewDocumentCommand{\citep}{O{} O{} m}{\cite{#3}}
\NewDocumentCommand{\citet}{O{} O{} m}{\cite{#3}}
\NewDocumentCommand{\citealp}{O{} O{} m}{\cite{#3}}

\title{Few Batches or Little Memory, But Not Both: Simultaneous Space and Adaptivity Constraints in Stochastic Bandits}
\date{}

\author{
Ruiyuan Huang$^{*}$ \\
School of Data Science, Fudan University \\
RuiyuanHuang00@gmail.com
\and
Zicheng Lyu$^{*}$ \\
School of Data Science, Fudan University \\
lyuzicheng@gmail.com
\and
Xiaoyi Zhu \\
School of Data Science, Fudan University \\
zhuxy22@m.fudan.edu.cn
\and
Zengfeng Huang$^{\dagger}$ \\
School of Data Science, Fudan University \\
Shanghai Innovation Institute \\
huangzf@fudan.edu.cn
}

\begin{document}

\maketitle

\let\thefootnote\relax
\footnotetext{$^{*}$Equal contribution.\quad $^{\dagger}$Corresponding author.}
\let\thefootnote\arabic

\begin{abstract}
We study stochastic multi-armed bandits under simultaneous constraints on space and adaptivity:
the learner interacts with the environment in $B$ batches and has only $W$ bits of persistent
memory.
Prior work shows that each constraint alone is surprisingly mild:
near-minimax regret $\widetilde{O}(\sqrt{KT})$ is achievable with $O(\log T)$ bits of memory under fully adaptive interaction, and with a $K$-independent $O(\log\log T)$-type number of batches when memory is unrestricted.
We show that this picture breaks down in the simultaneously constrained regime.
We prove that 
any algorithm with a $W$-bit memory constraint must use at least $\Omega(K/W)$ batches to achieve near-minimax regret $\widetilde{O}(\sqrt{KT})$, even under adaptive grids.
In particular, logarithmic memory rules out $O(K^{1-\varepsilon})$ batch complexity.

Our proof is based on an information bottleneck.
We show that near-minimax regret forces the learner to acquire $\Omega(K)$ bits of information about the hidden set of good arms under a suitable hard prior, whereas 
an algorithm with $B$ batches and $W$ bits of memory
allows only $O(BW)$ bits of information.
A key ingredient is a localized change-of-measure lemma that yields probability-level arm exploration guarantees, which is of independent interest.
We also give an algorithm that, for any bit budget $W$ with $\Omega(\log T) \le W \le O(K\log T)$, uses at most $W$ bits of memory and $\widetilde{O}(K/W)$ batches while achieving regret $\widetilde{O}(\sqrt{KT})$, nearly matching our lower bound up to polylogarithmic factors. 
\end{abstract}



\section{Introduction}

The stochastic multi-armed bandit (MAB) problem is a standard model for sequential decision making.
At each round $t\in[T]$, a learner chooses an arm $A_t\in[K]$ and observes a stochastic reward with mean $\mu_{A_t}$.
The goal is to minimize the expected pseudo-regret
\[
\mathbb{E}[\Reg_T]
:=
\mathbb{E}\!\left[\sum_{t=1}^T \bigl(\mu^\star-\mu_{A_t}\bigr)\right],
\qquad
\mu^\star := \max_{i\in[K]} \mu_i .
\]

Classical algorithms such as UCB \citep{auer2002finite} and Thompson sampling
\citep{thompson1933likelihood,agrawal2012analysis}
achieve the minimax rate $\widetilde{O}(\sqrt{KT})$.
These guarantees are typically established in an idealized model that implicitly grants two resources:
enough memory to maintain rich per-arm statistics, and full adaptivity, meaning that the learner may update its policy after every observation.

In many large-scale experimentation systems such as online advertising and recommendation systems, neither resource is always abundant.
When $K$ is large, storing per-arm state in fast memory can be costly.
At the same time, feedback is often processed asynchronously in system pipelines, so policy updates occur only periodically rather than after every interaction.

From a theoretical perspective, these two limitations correspond to constraints on two canonical computational resources: space and adaptivity.
A space constraint limits how many bits of memory a learner can use, while an adaptivity constraint limits how frequently a learner  may adapt its interaction with the environment based on past feedback.
Understanding how statistical performance depends on these resources, and whether one resource can compensate for the other, is therefore of fundamental interest.
In this paper, we study stochastic bandits under simultaneous constraints on both resources.

\paragraph{Space constraint.}
We model the space constraint by restricting the learner to $W$ bits of persistent memory, meaning that the state carried from one round to the next can take at most $2^W$ values.
Computation within a round is unrestricted; only the state retained between rounds is limited.
This is a natural model of space-constrained computation. In bandit theory, it dates back to the finite-memory bandit formulations of \citet{cover1968note,cover2003two}, and is also consistent with the modern sublinear-space viewpoint of \citet{liau2018constantspace}.
Under the standard word-RAM convention with word size $\Theta(\log T)$, $W=\Theta(\log T)$ corresponds to $O(1)$ machine words.

\paragraph{Adaptivity constraint.}
We consider the \emph{batched bandit} model~\citep{perchet2016batched,gao2019batched}, in which the horizon is partitioned into $B$ batches and outcomes are only incorporated between batches.
At the beginning of each batch, the learner must commit to the action-generation rule for that batch using only information available at the previous batch boundary.
Equivalently, feedback observed within a batch cannot affect later actions in the same batch.
Following \citet{gao2019batched}, we consider both the static-grid and the adaptive-grid variants. The static grid assumes that the batch boundaries are predetermined and fixed in advance, while the adaptive grid allows the learner to choose the next batch boundary using information available at the previous boundary.

Viewed separately, each constraint appears surprisingly mild.
Stochastic bandits admit near-optimal algorithms using only $O(\log T)$ bits of memory under fully adaptive interaction \citep{cover2003two,liau2018constantspace}, while batched bandits attain near-minimax regret with only $O(\log\log T)$ batches when memory is unrestricted \citep{perchet2016batched,gao2019batched,jin2021anytimebatched,esfandiari2021regret}.
These positive results rely on the \emph{other} resource being abundant:
space-bounded algorithms exploit frequent updates, whereas few-batch algorithms retain rich statistics across batch boundaries.

We consider the joint regime where the learner has a $W$-bit memory budget and interacts with the environment in $B$ batches.
Our results show that this regime is qualitatively different from either constraint in isolation:
once the space budget is logarithmic, near-minimax regret is no longer compatible with an $O(K^{1-\varepsilon})$ number of batches.

\begin{theorem}[Informal]
\label{thm:informal-lower}
For sufficiently large $T$, any $B$-batch stochastic bandit algorithm with $W$ bits of persistent memory and near-minimax regret $\widetilde{O}(\sqrt{KT})$ must satisfy
\[
B=\Omega(K/W),
\]
even if the batch grid is allowed to be adaptive.
\end{theorem}

A formal version of Theorem~\ref{thm:informal-lower} is given in \Cref{sec:lower_bound}. 

At a high level, the lower bound comes from an information bottleneck at batch boundaries.
All information gathered within a batch can influence future batches only through the $W$-bit state retained at the batch boundary.
With frequent feedback, a learner can recycle a small amount of memory over time; with abundant memory, it can compensate for few updates by storing accurate per-arm summaries.
Under simultaneous constraints, neither mechanism is available.
The learner must either explore many arms and retain only a coarse summary, or keep precise information about only a small subset of arms and therefore require many batches before covering all $K$ arms.
Our results show that this bottleneck fundamentally reshapes the batch--regret trade-off.

We also show that this lower bound is essentially tight.

\begin{theorem}[Informal]
\label{thm:informal-upper}
For every bit budget $\Omega(\log T) \le W \le O(K\log T)$
, there exists a stochastic static-grid bandit algorithm using at most $W$ bits of persistent memory and $\widetilde{O}(K/W)$ batches that achieves near-minimax regret
$\widetilde{O}(\sqrt{KT})$.
\end{theorem}
A formal version of Theorem~\ref{thm:informal-upper} is given in \Cref{sec:upper_bound}.
Here the $\Omega(\log T) \le W$ lower bound is needed to represent the relevant counters and empirical means to inverse-polynomial precision in $T$, while the $W \le O(K\log T)$ upper bound is without loss of generality since standard finite-armed bandit algorithms only need to store per-arm summaries for the $K$ arms.

\paragraph{Our contributions.}
We formalize stochastic bandits under simultaneous space and adaptivity constraints and show that limited space fundamentally increases the batch complexity of near-minimax learning.
Our main contributions are threefold.
\begin{itemize}
    \item \textbf{A general lower bound under limited memory.}
    We prove that, for sufficiently large $T$, any algorithm with a $W$-bit memory budget that achieves the standard near-minimax regret rate $\widetilde{O}(\sqrt{KT})$ must use at least $\Omega(K/W)$ batches, even under adaptive grids.
    
    \item \textbf{Two technical ingredients of independent interest.}
    In proving our lower bound, we establish two technical lemmas of independent interest: (i) a probability-level per-good-arm exploration lemma, and (ii) a localized change-of-measure lemma for under-sampling events, which lets the proof pay only for the first $n$ pulls of the perturbed arm.

    \item \textbf{A near-matching upper bound across memory budgets.}
    We propose a static-grid algorithm such that, for any bit budget $\Omega(\log T) \le W \le O(K\log T)$, it uses at most $W$ bits of persistent memory, at most $\widetilde{O}(K/W)$ batches, and achieves near-minimax regret $\widetilde{O}(\sqrt{KT})$.
\end{itemize}

\begin{remark}
    The adaptive-grid model has stronger algorithmic power than the static-grid model, because the learner can choose the batch boundaries adaptively.
    Our lower bound is proved for the algorithmically stronger adaptive-grid model, while our upper bound is achieved in the algorithmically weaker static-grid model. Therefore, in each setting, our result is established in the strongest possible sense.
    \end{remark}

\subsection{Technical Overview}
\label{subsec:tech_overview}

The introduction argued informally that simultaneous constraints on space and adaptivity create a bottleneck in how information can be accumulated and reused over time.
We now sketch how this intuition becomes a proof.
Our main technical contribution is a lower bound driven by this bottleneck.

\paragraph{A hard prior and the lower bound proof roadmap.}
For simplicity, assume that $K$ is even.
For each subset $S\subset[K]$ with $|S|=K/2$, let $\nu_S$ denote the Bernoulli instance in which arms in $S$ have mean $1/2$ and the remaining arms have mean $0$.
Let
\[
S^\star \subset [K]
\]
be drawn uniformly from all subsets of size $K/2$, and let $\nu_{S^\star}$ be the realized base instance.

On this prior, a near-minimax learner must satisfy two competing requirements.
First, it must avoid allocating a substantial exploration budget to bad arms, since every pull of a bad arm incurs regret $1/2$.
Second, it cannot afford to ignore too many good arms: for every good arm there is a nearby instance in which that arm becomes uniquely optimal, so uniform regret forces the learner to remain prepared for many such perturbations.
We capture this tension through a thresholded sampling profile $Y\in\{0,1\}^K$ that records which arms receive exploration on the natural per-arm scale $\widetilde{\Theta}(T/K)$.

The lower bound then has three steps.

We first show that bad arms rarely cross this threshold, while every good arm crosses it with constant probability. Equivalently, conditioned on $S^\star$, a constant fraction of the good arms cross in expectation.
This makes $Y$ a noisy coordinatewise predictor of the hidden good-arm set and yields a linear information lower bound,
\[
I(S^\star;Y)=\Omega(K).
\]
We then show that a $B$-batch protocol with only $W$ bits of persistent memory can transmit at most $O(BW)$ bits of instance-dependent information across batch boundaries.
Comparing the two bounds gives the lower bound. Below we explain our proof in more detail.

\paragraph{A thresholded sampling profile.}
The constant gap between good and bad arms is what makes this family useful for a regret lower bound.
Under any base instance $\nu_S$, every pull of a bad arm incurs regret $1/2$, whereas pulling a good arm incurs no regret.
Hence a near-minimax learner cannot afford to spend many rounds on bad arms and must concentrate most of its budget on good arms.
This makes it natural to distinguish arms according to whether they receive exploration on the average per-arm scale, namely about $\widetilde{\Theta}(T/K)$ pulls.
We therefore introduce a threshold
\[
n=\widetilde{\Theta}(T/K)
\]
and define the thresholded sampling profile
\[
Y_i := \ind\{N_i(T)\ge n\},\qquad Y\in\{0,1\}^K
\]
where $N_i(T)$ means the number of pulls of arm $i$ up to the horizon $T$.
Thus $Y_i$ records whether arm $i$ received a nontrivial share of the exploration budget.

A bound on false positives is immediate.
If a bad arm crosses the threshold, then the learner has already paid $\Theta(n)$ regret on that arm.
Hence low regret directly implies that not too many bad arms can satisfy $Y_i=1$.
The harder direction is to show that good arms cannot be ignored either.

\paragraph{Why false-negative control is the main difficulty.}
What is missing is the complementary statement that genuinely good arms must also cross the threshold with constant probability.
This is the delicate part of the proof.
Let $\nu_{S^\star}$ be the realized instance under this prior, and on that instance alone one might imagine that it suffices to find \emph{some} good arm and then exploit it. However, we prove that \emph{every} good arm must cross the threshold with constant marginal probability.
The key point is that the regret guarantee is uniform over all instances, not tailored to a single $\nu_S$.
For every $j\in S^\star$, we consider a neighboring instance $\nu_{S^\star}^{(j)}$ in which arm $j$ has mean $1/2+\Delta_0$, for some perturbation size $\Delta_0>0$ to be specified later, while all other arm means remain the same as in $\nu_{S^\star}$.
Thus arm $j$ becomes uniquely optimal under $\nu_{S^\star}^{(j)}$.
If the learner largely ignores arm $j$ under the base instance, then it is poorly positioned for $\nu_{S^\star}^{(j)}$ and would incur large regret there.
Hence no genuinely good arm can be safely ignored.

This is what turns the hard prior into a source of linear information.
The learner is not merely required to identify one arm worth exploiting; it must gather coordinatewise evidence about many membership bits of $S^\star$.

\paragraph{Why probability-level exploration requires localization.}
To make the previous intuition formal, we need a probability-level statement:
for every good arm $i$,
\[
\P_{\nu_S}(N_i(T)\ge n)\ge \beta
\]
for some constant $\beta>0$.
An expectation lower bound on $N_i(T)$ is too weak for this purpose.
Indeed, a learner could ignore arm $i$ on most trajectories and explore it extremely heavily only on a tiny exceptional set of runs.
Such bursty exploration can keep $\E[N_i(T)]$ at scale $n$ while making $\P(N_i(T)\ge n)$ arbitrarily small.
On the base instance $\nu_S$, this pathology is not ruled out by regret, because replacing pulls of one good arm by pulls of other good arms incurs no additional regret.

A natural route is therefore to argue that, on the perturbed instance $\nu_S^{(i)}$, near-minimax regret forces arm $i$ to be sampled at least $n$ times with sufficiently large probability, and then transfer this statement back to the base instance $\nu_S$.
The obstacle is that standard transcript-level change-of-measure bounds are poorly aligned with the under-sampling event
\[
E_i:=\{N_i(T)<n\}.
\]
Standard change-of-measure bounds compare the \emph{full} laws of the two experiments and therefore pay for all pulls of arm $i$, whereas the event $E_i$ depends only on whether the first $n$ pulls of arm $i$ ever occur.
A useful toy picture is the following: under the perturbed instance, the learner may use the first $n$ pulls of arm $i$ to discover that $i$ is special and then exploit $i$ for the remaining $T-n$ rounds.
The event $E_i$ is already decided after those first $n$ pulls, so all later exploitation pulls are irrelevant to $E_i$.
A standard full-transcript comparison nevertheless pays for all pulls of arm $i$, including those post-identification exploitation pulls, and can therefore incur a cost as large as order $T\Delta_0^2$, which is far too large for our parameter regime.

Our localized change-of-measure lemma truncates this comparison to the first $n$ pulls relevant to $E_i$, so the transfer cost is only of order $n\Delta_0^2$.
With the perturbation gap chosen at the nearly minimax-detectable scale
\[
\Delta_0=\widetilde{\Theta}\!\left(\sqrt{K/T}\right),
\]
we take
\[
n\asymp \frac{1}{\Delta_0^2}=\widetilde{\Theta}(T/K),
\]
so that
\[
n\Delta_0^2=O(1).
\]
This is exactly what allows the desired probability-level exploration guarantee to be transferred back from $\nu_S^{(i)}$ to $\nu_S$.
\paragraph{From thresholded exploration to an $\Omega(K)$ information requirement.}
The previous two steps give exactly the ingredients needed to interpret the thresholded profile $Y$ as a noisy classifier of the hidden membership bits
\[
X_i := \ind\{i\in S^\star\}.
\]
False positives are controlled because a bad arm that crosses the threshold already incurs $\Theta(n)$ regret, while false negatives are controlled because the localized perturbation argument shows that each good arm crosses the threshold with constant probability.
Consequently, the naive predictor $\hat X_i := Y_i$ has average error bounded away from $1/2$.
Equivalently, the average Bayes error for predicting $X_i$ from $Y_i$ is strictly smaller than random guessing by a constant margin.
A chain-rule / per-coordinate entropy argument then upgrades this coordinatewise predictive advantage into the global information lower bound
\[
I(S^\star;Y)=\Omega(K).
\]

\paragraph{From information to batches.}
The second half of the lower bound utilizes the batched nature of our algorithm. 
Once a batch starts, the full within-batch action rule is fixed using only the information available at that batch boundary.
Therefore the only instance-dependent information that can pass from one batch to the next is what is retained in persistent memory at the boundaries.
We show that the sampling profile $Y$ is a deterministic function of the seed and the boundary-memory transcript; in the adaptive-grid case, the realized batch boundaries are themselves recursively determined by the same objects.
Since the seed is independent of $S^\star$ and each boundary state carries at most $W$ bits, the entire boundary-memory transcript has entropy at most $O(BW)$.
A data-processing argument therefore yields
\[
I(S^\star;Y)\le O(BW).
\]
Comparing this with the lower bound on the same quantity gives
\[
B=\Omega(K/W)
\]
in the large-horizon regime.

\paragraph{Upper bound intuition.}
The upper bound matches the same bottleneck from the algorithmic side.
Instead of storing a large active set of candidate arms, the algorithm keeps only an incumbent arm, a small summary of its performance, and a block of at most $S$ challenger arms at a time.
These challengers are compared with the incumbent through a multistage test, after which the algorithm moves on to the next block.
This blockwise scan explicitly trades memory for adaptivity:
with block size $S$, the implementation uses $O(S\log T)$ bits of memory and the number of batches drops to $\widetilde{O}(K/S)$.
Equivalently, for any bit budget $W$ with $\Omega(\log T) \le W \le O(K\log T)$, choosing $S\asymp W/\log T$ gives an implementation using at most $W$ bits and $\widetilde{O}(K/W)$ batches.
The resulting algorithm nearly matches the lower bound up to polylogarithmic factors. 

Note that the  $W \le O(K\log T)$ upper bound is without loss of generality since standard finite-armed bandit algorithms only need to store per-arm summaries for the $K$ arms.

\paragraph{Organization.}
After discussing related work in Section~\ref{sec:related_work}, Section~\ref{sec:preliminaries} formalizes the space-bounded batched bandit model.
Section~\ref{sec:lower_bound} proves the lower bound in three steps: a probability-level exploration argument via localized one-arm perturbations, an information-acquisition lower bound, and a batch-memory bottleneck bound.
Section~\ref{sec:upper_bound} presents a nearly matching upper bound.
Section~\ref{sec:conclusion} concludes with open directions.
\section{Related Work}
\label{sec:related_work}

Our work lies at the intersection of two lines of research in bandit theory:
\emph{limited adaptivity} and \emph{space-bounded learning}.
We briefly review these two literatures and then explain how our setting differs from both.

\paragraph{Limited adaptivity.}
A large literature studies bandits under restrictions on how often the learner may update its policy.
For stochastic batched bandits, early results include \citet{perchet2016batched} for the two-armed case and \citet{gao2019batched}, which developed the BaSE policy and established near-optimal guarantees and lower bounds for general $K$-armed problems.
Subsequent work extended these ideas to broader settings, including stochastic, linear, and adversarial bandits \citep{esfandiari2021regret}, Thompson-sampling-based batched methods \citep{kalkanli2021batchedthompson}, and anytime batched algorithms with near-optimal $O(\log\log T)$-type batch complexity \citep{jin2021anytimebatched}.
Related guarantees are also known in structured settings such as finite-action linear contextual bandits \citep{han2020sequentialbatch,ruan2021limitedadaptivity}, high-dimensional sparse linear contextual bandits \citep{ren2024dynamicbatch}, generalized linear contextual bandits \citep{ren2020batchedglm,sawarni2024limitedadaptivity}, and large-action-space contextual linear bandits \citep{hanna2023efficient,ren2024optimalbatchedlinear}.
Related formulations restrict adaptivity through switching costs, delayed feedback, or more general round-complexity constraints \citep{arora2012policyregret,cesabianchi2013switchingcosts,joulani2013delayed,vernade2020delayedlinear,agarwal2017limitedadaptivity,jun2016toparmbatch}.
However, these works typically bound how often the policy may change while allowing the learner to retain arbitrarily rich information between updates.

\paragraph{Space-bounded learning.}
Learning with limited memory has been studied broadly in learning theory~\citep{Raz16space,Raz17space,Moshkovitz17space,Moshkovitz17generalspace}.
Within bandits, classical antecedents on finite-memory models go back to \citet{cover1968note,cover2003two}, and modern stochastic bandit results show that one can achieve strong regret guarantees with constant or logarithmic space under fully sequential interaction \citep{liau2018constantspace}.
Related memory-limited variants include memory-constrained adversarial bandits \citep{xu2021memoryconstrained}, limited-memory approaches based on subsampling or forgetting \citep{baudry2021limitedmemory}, and bounded arm-memory models in which the learner may keep statistics for only a limited number of arms at a time \citep{chaudhuri2020boundedarmmemory,maiti2021boundedarmmemory}.

A different notion appears in streaming bandits, where arms arrive in a stream and the learner is constrained by pass and storage limitations \citep{assadi2020exploration,jin2021optimalstreaming,agarwal2022sharp,assadi2022singlepasslowerbounds,wang2023singlepass,li2023tightmemoryregret,assadi2024bestarmevades}.
These models are quite different from ours. In the streaming formulations above, the available arm set changes over time, whereas in our setting the arm set is fixed and always accessible. Moreover, streaming algorithms are typically fully adaptive within a pass, while in our batched model the learner cannot use feedback observed within a batch to change later actions in that batch. This distinction is also reflected in the complexity trade-off: in streaming bandits, one can recover the optimal regret rate up to logarithmic factors with constant space and only $O(\log\log T)$ passes, whereas in our batched model, recovering the same rate with $O(\log T)$ bits of memory requires $\widetilde \Omega(K)$ batches. Furthermore, techniques used in proving lower bounds are fundamentally different between streaming bandits and our setting.

\paragraph{Our setting.}
The works closest in spirit to ours are batched bandits and finite-memory bandits.
Relative to the batched-bandit literature, we additionally impose an explicit bound on the persistent memory retained across batch boundaries.
Relative to the space-bounded bandit literature, we additionally restrict adaptivity by allowing policy updates only between batches.
We study this joint regime and show that combining the two constraints creates a qualitatively stronger bottleneck than either constraint alone.

\section{Preliminaries}
\label{sec:preliminaries}

\subsection{Notation}
For $n\ge 1$, write $[n]:=\{1,\dots,n\}$, and let $\ind\{E\}$ denote the indicator of an event $E$.
We use $\log$ for the natural logarithm and $\log_2$ for the base-$2$ logarithm.
For a bandit instance $\nu=(P_1,\dots,P_K)$, let $\P_\nu$ and $\E_\nu$ denote probability and expectation under the law induced by $\nu$ and the algorithm; when the instance itself is random, we write $\P$ and $\E$ for the corresponding mixture law.
We use $\widetilde{O}(\cdot)$, $\widetilde{\Omega}(\cdot)$ and $\widetilde{\Theta}(\cdot)$ to suppress polylogarithmic factors in $K$ and $T$, and write $H(\cdot)$ and $I(\cdot;\cdot)$ for Shannon entropy and mutual information, measured in bits.

\subsection{Stochastic multi-armed bandits}
\label{subsec:bandit_model}

We consider a $K$-armed stochastic bandit with horizon $T$.
Arm $i\in[K]$ has reward distribution $P_i$ supported on $[0,1]$ and mean
$\mu_i := \E_{X\sim P_i}[X]$.
At each round $t\in[T]$, the learner chooses an arm $A_t\in[K]$ and observes a reward
$R_t\sim P_{A_t}$; conditional on the instance, rewards are independent across pulls.
Let
\[
\mu^\star:=\max_{i\in[K]}\mu_i,
\qquad
\Reg_T := \sum_{t=1}^T (\mu^\star-\mu_{A_t}).
\]
We measure performance by the expected pseudo-regret $\E_\nu[\Reg_T]$.

\subsection{Space-bounded bandits}
\label{subsec:space_model}

A learner with a $W$-bit persistent-memory budget maintains a persistent state
$M_t\in\mathcal M$ after each round, where
\[
|\mathcal M|\le 2^W.
\]
The algorithm may use a random seed $U$, sampled once at time $0$ and independent of the bandit instance; $U$ is not counted toward the space budget.
When convenient, all algorithmic randomization is absorbed into $U$.
We do not restrict computation time, only the cross-round state.

In the fully sequential model, a $W$-bit learner is specified by measurable maps such that, for each round $t\in[T]$,
\[
A_t = \rho_t(M_{t-1},U),
\qquad
M_t = \phi_t(M_{t-1},U,A_t,R_t),
\]
with fixed initial state $M_0\in\mathcal M$.

\subsection{Batched bandits}
\label{subsec:batch_model}
Let
\[
\mathcal F_t := \sigma(U,A_1,R_1,\dots,A_t,R_t),
\qquad t=0,1,\dots,T,
\]
with the convention \(\mathcal F_0:=\sigma(U)\),
denote the natural filtration of the interaction.
In a $B$-batch protocol, the horizon is partitioned into consecutive batches
\[
0=t_0<t_1<\cdots<t_B=T,
\]
where batch $b\in[B]$ consists of rounds $\{t_{b-1}+1,\dots,t_b\}$.

The batching constraint is that, at time $t_{b-1}$, the learner must choose the next boundary $t_b$ and the entire within-batch action rule for rounds $t_{b-1}+1,\dots,t_b$ as $\mathcal F_{t_{b-1}}$-measurable objects.
Equivalently, rewards observed during batch $b$ cannot affect later actions in the same batch.

We consider two variants.
(i) In the \emph{static-grid} model, the full grid $(t_1,\dots,t_B)$ is fixed before interaction, equivalently, it is $\mathcal F_0$-measurable.
(ii) In the \emph{adaptive-grid} model, each $t_b$ is $\mathcal F_{t_{b-1}}$-measurable, but must still be chosen before batch $b$ begins.

\subsection{The combined model: space-bounded and batched bandits}
\label{subsec:memory_model}

We now combine batching with the $W$-bit memory constraint.
The learner maintains a persistent state $M_t\in\mathcal M$ with
$|\mathcal M|\le 2^W$.
All algorithmic randomness is absorbed into a seed $U$, sampled once at time $0$ and not counted toward the memory budget.
The current time index is public and likewise not counted toward the memory budget.
Thus, at the start of batch $b$, the learner's batch boundary and within-batch action plan may depend only on
$(t_{b-1},M_{t_{b-1}},U),$
which serves as a memory-constrained summary of the full pre-boundary history $\mathcal F_{t_{b-1}}$.

Formally, a $B$-batch, $W$-bit algorithm is specified by measurable maps as follows.
For each internal boundary $b\in[B-1]$,
\[
t_b=\psi_b(t_{b-1},M_{t_{b-1}},U),
\qquad
t_{b-1}<t_b<T,
\]
and we set $t_B:=T$.
For each batch $b\in[B]$ and round $t\in\{t_{b-1}+1,\dots,t_b\}$,
\[
A_t = \rho_{b,t}(t_{b-1},M_{t_{b-1}},U),
\qquad
M_t=\phi_t(M_{t-1},U,A_t,R_t).
\]

Rewards may be observed and the state $M_t$ updated within a batch, but this does not create within-batch adaptivity, since the maps $\rho_{b,t}$ are fixed at time $t_{b-1}$.
Hence information gathered during batch $b$ can influence later batches only through the boundary state $M_{t_b}$.

In the static-grid model, the internal boundaries $t_1,\dots,t_{B-1}$ are fixed before interaction (equivalently, they may depend on $U$ but not on observed rewards).
In the adaptive-grid model, the realized internal grid
$
\tau:=(t_1,\dots,t_{B-1})$
is recursively determined by $U$ and the boundary states
$M_{t_0},M_{t_1},\dots,M_{t_{B-1}}$.
We will use this fact in Section~\ref{subsec:batch_complexity} when bounding the information that can be transmitted across batches.

Finally, this formulation is at least as permissive as the classical delayed-feedback batched model for lower-bound purposes: rewards may be observed and the internal state updated within a batch, but such information still cannot affect later actions in the same batch.
Therefore any lower bound proved here also applies to the classical delayed-feedback batched setting.
\section{Lower Bounds on Batch Complexity under Space Constraints}
\label{sec:lower_bound}

This section proves the formal lower bound stated in Theorem~\ref{thm:batch_lower_bound_rigorous}.
We establish the result directly for the more permissive adaptive-grid model; the same conclusion therefore holds \emph{a fortiori} for static-grid algorithms, which form a special case.
At a high level, we construct a Bernoulli hard family in which half of the arms are good and half are bad, and study a thresholded sampling profile that records which arms are pulled on the natural per-arm scale.
The proof has three steps.
First, a regret argument together with a localized one-arm perturbation argument shows that bad arms cannot cross the threshold too often, while each good arm must cross it with constant probability.
Second, this makes the thresholded profile carry linear information about the hidden good-arm set.
Third, batching with a $W$-bit persistent-memory budget allows at most $O(BW)$ bits of instance-dependent information to pass across batch boundaries.
Comparing the two bounds yields
\[
B=\Omega(K/W).
\]

We now formalize this argument.

\paragraph{Bernoulli subclass.}
For $p\in[0,1]$, let $\Ber(p)$ denote the Bernoulli distribution on $\{0,1\}$ with mean $p$.
Since the hard family used below is Bernoulli, it suffices to prove the lower bound on
\[
\mathcal V_{\mathrm{Ber}}
:=
\bigl\{(\Ber(\mu_1),\dots,\Ber(\mu_K)):\ \mu_1,\dots,\mu_K\in[0,1]\bigr\}.
\]
Any regret guarantee over a larger class of stochastic bandit instances, such as all $[0,1]$-valued reward distributions, in particular implies the same guarantee on $\mathcal V_{\mathrm{Ber}}$.

\paragraph{Standing notation for this section.}
Throughout this section, we write
\[
\mathbf M:=(M_{t_1},\dots,M_{t_{B-1}})
\]
for the collection of persistent memory states at the batch boundaries. This is
only a notational device: we do not mean that the algorithm explicitly stores
the entire tuple \(\mathbf M\) simultaneously.
We also write $N_i(t)$ for the number of pulls of arm $i$ up to time $t$, and $N_i:=N_i(T)$.

\begin{reptheorem}[Formal]{thm:informal-lower}
\label{thm:batch_lower_bound_rigorous}
Suppose $K$ is even.
There exist universal constants $c,\gamma_0>0$ such that the following holds.

Let $\mathcal A$ be any adaptive-grid $B$-batch stochastic bandit algorithm with persistent state space $\mathcal M$ satisfying
$|\mathcal M|\le 2^W$.
If, for some constant $C\ge 1$,
\begin{equation}\label{eq:reg-def}
    \sup_{\nu\in\mathcal V_{\mathrm{Ber}}}\E_\nu[\Reg_T]
\le C\sqrt{KT}\log T\log K,
\end{equation}
with $T \ge c\, C^6 K\log^6 T\log^6 K$, then
\[
B \ge \frac{\gamma_0 K - \frac12\log_2(2K)}{W}.
\]
In particular,
\[
B=\Omega(K/W).
\]
\end{reptheorem}

\Cref{thm:batch_lower_bound_rigorous} applies to any $\tilde{O}(\sqrt{KT})$ regret bound with minor modifications; we adopt the concrete rate $O(\sqrt{KT}\log T\log K)$ here to exactly match the known batched-bandit upper bound~\citep{gao2019batched}.
Thus, the theorem shows that even attaining this rate under a $W$-bit persistent-memory budget requires batch complexity at least on the order of $K/W$.

\begin{definition}[Bernoulli hard family and thresholded sampling profile]
\label{def:hard_family_profile}
Assume $K$ is even, and let
\[
\mathcal S_{K/2}:=\{S\subset[K]: |S|=K/2\}.
\]
For each $S\in\mathcal S_{K/2}$, let $\nu_S\in\mathcal V_{\mathrm{Ber}}$ denote the bandit instance with arm distributions
\[
P_i=
\begin{cases}
\Ber(1/2), & i\in S,\\
\Ber(0), & i\notin S.
\end{cases}
\]
Equivalently, the arm means satisfy
\[
\mu_i=
\begin{cases}
1/2, & i\in S,\\
0, & i\notin S.
\end{cases}
\]
We place the prior
\[
S^\star \sim \mathrm{Unif}(\mathcal S_{K/2})
\]
on this family.

Given a threshold $n$, define the thresholded sampling profile
\[
Y_i := \ind\{N_i(T)\ge n\},
\qquad
Y=(Y_1,\dots,Y_K)\in\{0,1\}^K.
\]
\end{definition}

For $S\in\mathcal S_{K/2}$, $j\in S$, and $\Delta>0$, let
\[
\nu_S^{(j,\Delta)}
\]
denote the one-arm perturbation of $\nu_S$ obtained by replacing the reward law of arm $j$ by $\Ber(1/2+\Delta)$ and leaving all other arms unchanged.
When the perturbation size is fixed, we abbreviate this as $\nu_S^{(j)}$.

\paragraph{Proof roadmap.}
We now spell out the three ingredients announced above.
For readability, the third step is stated below as a coarse $BW$ upper bound, although the proof actually yields the slightly sharper $(B-1)W$ bound.

\begin{enumerate}
  \item[(A)] \textbf{Per-good-arm exploration (Section~\ref{subsec:necessary_exploration}).}
  We prove that for every $S\in\mathcal S_{K/2}$ and every good arm $j\in S$,
  \[
  \P_{\nu_S}(N_j(T)\ge n)\ge \beta
  \]
  for an absolute constant $\beta>0$.
  This is the only step that uses the one-arm perturbation family explicitly.

  \item[(B)] \textbf{Information lower bound (Section~\ref{subsec:info_lower_bound}).}
  We combine the false-negative control from Step~(A) with the false-positive control that comes directly from regret.
  Together these imply that the thresholded profile $Y$ is a noisy coordinatewise predictor of the hidden good-arm set, and hence
  \[
  I(S^\star;Y)=\Omega(K).
  \]

  \item[(C)] \textbf{From information to batches (Section~\ref{subsec:batch_complexity}).}
  We show that $Y$ is determined by the information available at batch boundaries, namely the seed and the persistent boundary states.
  Applying data processing together with
  \[
  H(\mathbf M)\le BW
  \]
  yields
  \[
  I(S^\star;Y)\le BW.
  \]
  In the adaptive-grid case, the realized batch boundaries are themselves determined by the same objects, so they do not create an additional instance-dependent information channel.
\end{enumerate}

Equivalently, the quantitative core of the proof is the information chain
\begin{equation}
\label{eq:main_info_chain}
\boxed{
\underbrace{\Omega(K)}_{\text{Step (B): information forced by the hard family}}
\ \le\ I(S^\star;Y)
\ \le\ \underbrace{BW}_{\text{Step (C): batch-memory bottleneck}}
}\,.
\end{equation}
Combining Steps~(B) and~(C) yields Theorem~\ref{thm:batch_lower_bound_rigorous}.

\subsection{An Event-Restricted Change of Measure and Per-Good-Arm Exploration}
\label{subsec:necessary_exploration}

This subsection has a local role in the proof and a broader methodological one.
Its immediate purpose is to establish a probability-level exploration guarantee for good arms in the hard family:
under every base instance $\nu_S$, each good arm must cross a prescribed sampling threshold with constant probability.
This will be the only output of the perturbation argument used later in the information-theoretic part of the lower bound.

At the same time, the underlying change-of-measure argument is not specific to our Bernoulli hard family.
For independent interest, we state it in a reusable adaptive-sampling form.
To keep the main proof focused on the lower-bound argument, the full measurable-space formulation and the associated truncation-based proofs are deferred to Appendix~\ref{app:localized_com}.

The conceptual point is the following.
Classical transcript-level change-of-measure arguments charge the discrepancy between two environments to the full amount of information accumulated by the algorithm.
For under-sampling events, however, this is often poorly aligned with the event itself.
An event such as $\{N_j(T)\le n\}$ is insensitive to any reward from the perturbed arm beyond its $n$-th pull.
Thus the natural comparison cost should scale with the part of the perturbed stream that the event can actually depend on, rather than with the entire transcript.
The results below formalize exactly this prefix-sensitive, event-visible comparison principle.

We then instantiate this general tool for one-arm perturbations of the hard family and obtain the following exploration lemma.

\begin{lemma}[Per-good-arm exploration]
\label{lem:per_good_arm_exploration}
In the hard family of Definition~\ref{def:hard_family_profile}, assume that
algorithm $\mathcal A$ satisfies the regret bound \eqref{eq:reg-def} for some
constant $C>0$.
There exist absolute constants $c_\Delta,\alpha,\beta>0$ such that the following holds.
Define
\[
\Delta_0 := c_\Delta C\log(T)\log(K)\sqrt{\frac{K}{T}},
\qquad
n:=\left\lfloor \alpha \frac{T}{C^2K\log^2(T)\log^2(K)}\right\rfloor.
\]
Assume $T$ is large enough so that $\Delta_0\le 1/4$ and $1\le n\le T/2$.
Then for every $S\in\mathcal S_{K/2}$ and every $j\in S$,
\[
\P_{\nu_S}\bigl(N_j(T)\ge n\bigr)\ge \beta.
\]
\end{lemma}

One admissible choice is
\[
c_\Delta=8,\qquad \alpha=\frac{\ln(1.8)}{512},\qquad \beta=0.1.
\]
The proof below verifies that these values work; the subsequent arguments only use that
$c_\Delta,\alpha,\beta$ are absolute positive constants.

Fix once and for all any admissible constants $c_\Delta,\alpha,\beta>0$
for which Lemma~\ref{lem:per_good_arm_exploration} holds, and define
\[
\Delta_0 := c_\Delta C\log(T)\log(K)\sqrt{\frac{K}{T}},
\qquad
n:=\left\lfloor \alpha \frac{T}{C^2K\log^2(T)\log^2(K)}\right\rfloor.
\]
All subsequent statements in this section are with respect to this fixed admissible choice.

The proof of Lemma~\ref{lem:per_good_arm_exploration} proceeds by instantiating the next three general results,
which form a reusable event-restricted change-of-measure tool for adaptive sampling problems with one-stream perturbations.

\subsubsection{An Event-Restricted Change of Measure via \texorpdfstring{$\ChiSq$}{ChiSq}-Divergence}

Many bandit lower bounds compare two environments that differ only on a single arm \(j\) and transfer probability statements between the corresponding transcript laws
\citep[see, e.g.,][]{kaufmann2016complexity,garivier2016optimal,lattimore2020banditalgorithms}.
Standard Kullback--Leibler- or total-variation-based arguments are transcript-level: they charge the comparison to all pulls of the perturbed arm.
For under-sampling events, this is often fundamentally misaligned.
If the event of interest is \(\{N_j(T)\le n\}\), then only the first \(n\) rewards of arm \(j\) can affect that event; anything observed from arm \(j\) after the \(n\)-th pull is irrelevant to it.

The key distinction is therefore between transcript-level information and event-visible information.
The result below isolates a reusable change-of-measure principle whose penalty is aligned with the latter.
More precisely, whenever the event of interest depends on at most the first \(n\) observations from the perturbed stream, the comparison cost is governed by \(n\) one-step \(\ChiSq\) factors.
We do not present this as a new global divergence inequality.
Rather, it is an event-restricted and prefix-sensitive packaging of standard likelihood-ratio and second-moment arguments.

\begin{definition}[$\ChiSq$-divergence]
For probability measures $P\ll Q$, define
\[
\ChiSq(P\|Q)
:=
\int \left(\frac{dP}{dQ}-1\right)^2\,dQ
=
\E_{\xi\sim Q}\!\left[\left(\frac{dP}{dQ}(\xi)-1\right)^2\right].
\]
Equivalently,
\[
\E_{\xi\sim Q}\!\left[\left(\frac{dP}{dQ}(\xi)\right)^2\right]
=
1+\ChiSq(P\|Q).
\]
\end{definition}

The identity above is exactly what makes $\ChiSq$ the natural divergence here:
the second moment of the prefix likelihood ratio over the first $n$ pulls factorizes as
\[
\bigl(1+\ChiSq(P\|Q)\bigr)^n.
\]

The full formal setup appears in Appendix~\ref{app:localized_com}.
For the statements below, it suffices to consider two adaptive-sampling environments
$\P_0$ and $\P_1$ induced by the same non-anticipating policy and differing only on one arm $j$,
whose reward laws are $P$ under $\P_0$ and $Q$ under $\P_1$.
Write
\[
\mathcal F_t:=\sigma(U,A_1,R_1,\dots,A_t,R_t)
\]
for the transcript filtration, and for each integer $n\ge 0$ let
\[
\mathcal G_n
:=
\sigma\!\Bigl(U,\ (X_{i,\ell})_{i\neq j,\ \ell\ge1},\ X_{j,1},\dots,X_{j,n}\Bigr).
\]

Lemma~\ref{lem:budget_event_measurable}, proved in Appendix~\ref{app:localized_com}, shows that under a pull budget $N_j(T)\le n$, the event in question is already visible in $\mathcal G_n$.

\begin{proposition}[Prefix-measurable $\ChiSq$ change of measure]
\label{prop:prefix_chisq}
Assume $\ChiSq(P\|Q)<\infty$.
Then for every integer $n\ge0$ and every event $E\in\mathcal G_n$,
\[
\P_0(E)
\le
\sqrt{\P_1(E)}\,
\bigl(1+\ChiSq(P\|Q)\bigr)^{n/2}
\le
\sqrt{\P_1(E)}\,
\exp\!\left(\frac{n}{2}\ChiSq(P\|Q)\right).
\]
\end{proposition}

\begin{proof}
Let
\[
\P_k^{(n)}:=\P_k|_{\mathcal G_n},\qquad k\in\{0,1\},
\]
denote the restrictions of $\P_0$ and $\P_1$ to $\mathcal G_n$.
Define
\[
L_n:=\prod_{\ell=1}^{n}\frac{dP}{dQ}(X_{j,\ell}),
\qquad
L_0:=1.
\]
Under both $\P_0$ and $\P_1$, the random element
\[
\Bigl(U,\ (X_{i,\ell})_{i\neq j,\ \ell\ge1}\Bigr)
\]
has the same law and is independent of $(X_{j,1},\dots,X_{j,n})$.
Moreover,
\[
(X_{j,1},\dots,X_{j,n})\sim P^{\otimes n}\ \text{under }\P_0,
\qquad
(X_{j,1},\dots,X_{j,n})\sim Q^{\otimes n}\ \text{under }\P_1.
\]
Hence $\P_0^{(n)}\ll \P_1^{(n)}$ and
\[
\frac{d\P_0^{(n)}}{d\P_1^{(n)}}=L_n
\qquad \P_1\text{-a.s. on }\mathcal G_n.
\]
Equivalently, for every bounded $\mathcal G_n$-measurable random variable $H$,
\[
\E_0[H]=\E_1[H\,L_n].
\]
Applying this identity with $H=\ind_E$ yields
\[
\P_0(E)=\E_1[\ind_E L_n].
\]
By Cauchy--Schwarz,
\[
\P_0(E)
\le
\sqrt{\P_1(E)}\,
\sqrt{\E_1[L_n^2]}.
\]
Under $\P_1$, the variables $X_{j,1},\dots,X_{j,n}$ are i.i.d.\ with law $Q$, so
\[
\E_1[L_n^2]
:=
\prod_{\ell=1}^{n}
\E_{\xi\sim Q}\!\left[\left(\frac{dP}{dQ}(\xi)\right)^2\right]
=
\bigl(1+\ChiSq(P\|Q)\bigr)^n.
\]
This proves the first inequality.
The second follows from $1+x\le e^x$.
\end{proof}

\begin{lemma}[Event-restricted $\ChiSq$ change of measure for adaptive sampling]
\label{lemma:localized_change_of_measure}
Assume $\ChiSq(P\|Q)<\infty$.
Then for every event $E\in\mathcal F_T$ and every integer $n\ge0$,
\[
\P_0\!\bigl(E\cap\{N_j(T)\le n\}\bigr)
\le
\sqrt{\P_1\!\bigl(E\cap\{N_j(T)\le n\}\bigr)}\,
\bigl(1+\ChiSq(P\|Q)\bigr)^{n/2}.
\]
Consequently,
\[
\P_0\!\bigl(E\cap\{N_j(T)\le n\}\bigr)
\le
\sqrt{\P_1(E)}\,
\exp\!\left(\frac{n}{2}\ChiSq(P\|Q)\right).
\]
\end{lemma}

\begin{proof}
By Lemma~\ref{lem:budget_event_measurable},
\[
E\cap\{N_j(T)\le n\}\in\mathcal G_n.
\]
Apply Proposition~\ref{prop:prefix_chisq} to the event
\[
E\cap\{N_j(T)\le n\}.
\]
The second inequality follows from
\[
\P_1\!\bigl(E\cap\{N_j(T)\le n\}\bigr)\le \P_1(E)
\]
and
\[
1+x\le e^x.
\]
\end{proof}

\begin{corollary}[Specialization to events already enforcing the budget]
\label{cor:blackbox_local_chisq}
If $E\in\mathcal F_T$ satisfies $E\subseteq\{N_j(T)\le n\}$, then
\[
\P_0(E)
\le
\sqrt{\P_1(E)}\,
\bigl(1+\ChiSq(P\|Q)\bigr)^{n/2}
\le
\sqrt{\P_1(E)}\,
\exp\!\left(\frac{n}{2}\ChiSq(P\|Q)\right).
\]
In particular, if $\P_1(E)\le\delta$, then
\[
\P_0(E)
\le
\sqrt{\delta}\,
\exp\!\left(\frac{n}{2}\ChiSq(P\|Q)\right).
\]
Equivalently,
\[
\P_1(E)
\ge
\P_0(E)^2\,
\bigl(1+\ChiSq(P\|Q)\bigr)^{-n}
\ge
\P_0(E)^2\,
\exp\!\bigl(-n\,\ChiSq(P\|Q)\bigr).
\]
\end{corollary}

\begin{proof}
If $E\subseteq\{N_j(T)\le n\}$, then
\[
E=E\cap\{N_j(T)\le n\}.
\]
Applying Lemma~\ref{lemma:localized_change_of_measure} and rearranging gives the desired inequality.
\end{proof}

\paragraph{How the general tool will be used.}
In our application, \(E\) will be an under-sampling event for a single arm \(j\), and \((\P_0,\P_1)\) will correspond to a base instance and a one-arm perturbed instance.
The point of Corollary~\ref{cor:blackbox_local_chisq} is that once one proves under the perturbed environment that \(\P_1(E)\) is small, the same event can be transferred back to the base environment with a penalty that depends only on the first \(n\) relevant observations from the modified arm.
This is exactly the scale needed for the exploration guarantee below. More broadly, Lemma~\ref{lemma:localized_change_of_measure} applies whenever one compares two adaptive sampling environments that differ only through a single data stream and the event of interest is controlled by a bounded pull budget on that stream.
Thus its relevance is not restricted to the present Bernoulli hard family.

\begin{remark}[Why the event-restricted form matters]
\label{rem:replace_kl_by_local_chisq}
Classical transcript-level change-of-measure arguments compare the \emph{full} laws and yield, for example,
\[
\KL(\P_0\|\P_1)=\E_0[N_j(T)]\,\KL(P\|Q),
\qquad
\KL(\P_1\|\P_0)=\E_1[N_j(T)]\,\KL(Q\|P),
\]
and analogous bounds for other global divergences
\citep{bretagnolle1979estimation,tsybakov2008nonparametric,lattimore2020banditalgorithms}.
These quantities pay for all pulls of arm \(j\) under one of the two environments.

For under-sampling events such as \(\{N_j(T)\le n\}\), this can be much larger than the relevant event-level budget.
By contrast, Lemma~\ref{lemma:localized_change_of_measure} and Corollary~\ref{cor:blackbox_local_chisq}
pay only for the first \(n\) rewards of arm \(j\) that are relevant to the event under consideration,
independently of how often arm \(j\) may be pulled on typical trajectories under either environment.
This event-level alignment is precisely what is needed in the application below to turn regret information
under a perturbed instance into a probability-level exploration guarantee.
\end{remark}

\subsubsection{Proof of Lemma~\ref{lem:per_good_arm_exploration}}

\begin{proof}
It suffices to verify the claim for one admissible choice of constants.
Take
\[
c_\Delta=8,\qquad \alpha=\frac{\ln(1.8)}{512},\qquad \beta=0.1.
\]
Fix $S\in\mathcal S_{K/2}.$ and $j\in S$, and define
\[
\Delta_0 := c_\Delta C\log(T)\log(K)\sqrt{\frac{K}{T}},
\qquad
n:=\left\lfloor \alpha \frac{T}{C^2K\log^2(T)\log^2(K)}\right\rfloor.
\]
Define the instance $\nu_S^{(j)}$ by changing only arm $j$ to mean
$\mu_j=1/2+\Delta_0$, keeping all other arms identical to $\nu_S$.
Let $E_j:=\{N_j(T)<n\}$.

Under $\nu_S^{(j)}$, arm $j$ is the unique optimal arm with mean $1/2+\Delta_0$.
For any $i\neq j$, we have $\mu_i\le 1/2$, hence each round not pulling $j$ incurs
instantaneous expected regret at least $\Delta_0$.
Therefore,
\[
\Reg_T(\nu_S^{(j)}) \ \ge\ \Delta_0\big(T-N_j(T)\big).
\]
On $E_j$, $T-N_j(T)\ge T-n\ge T/2$, so
\[
\E_{\nu_S^{(j)}}[\Reg_T]
\ \ge\ \E_{\nu_S^{(j)}}\!\big[\Delta_0(T-N_j(T))\ind_{E_j}\big]
\ \ge\ \frac{\Delta_0 T}{2}\,\P_{\nu_S^{(j)}}(E_j).
\]
By the regret bound \eqref{eq:reg-def}, $\E_{\nu_S^{(j)}}[\Reg_T]\le C\log(T)\log(K)\sqrt{KT}$, hence
\[
\P_{\nu_S^{(j)}}(E_j)
\ \le\ \frac{2C\sqrt{KT}\log(T)\log(K)}{\Delta_0 T}
\ =\ \frac14.
\]

The environments $\nu_S$ and $\nu_S^{(j)}$ differ only at arm $j$.
At arm $j$, the reward law is $P=\mathrm{Ber}(1/2)$ under $\nu_S$ and
$Q=\mathrm{Ber}(1/2+\Delta_0)$ under $\nu_S^{(j)}$.
For $|\Delta_0|\le 1/4$,
\[
\ChiSq(P\|Q)
=\frac{4\Delta_0^2}{1-4\Delta_0^2}
\le 16\Delta_0^2.
\]
Since $E_j\subseteq\{N_j(T)\le n\}$, Lemma~\ref{lemma:localized_change_of_measure} gives
\[
\P_{\nu_S}(E_j)
\le \sqrt{\P_{\nu_S^{(j)}}(E_j)}\exp\!\Big(\frac n2\ChiSq(P\|Q)\Big)
\le \frac12 \exp(8n\Delta_0^2).
\]
By the definition of $n$ and $\Delta_0$, we have
\[
8n\Delta_0^2
\le
8\Bigl(\alpha\frac{T}{C^2K\log^2(T)\log^2(K)}\Bigr)
\Bigl(64C^2\log^2(T)\log^2(K)\frac{K}{T}\Bigr)
=
512\alpha
=
\ln(1.8).
\]
Thus
\[
\P_{\nu_S}(E_j)
\le
\frac12 \exp(8n\Delta_0^2)
\le
\frac12 e^{\ln(1.8)}
=0.9.
\]
Equivalently,
\[
\P_{\nu_S}\bigl(N_j(T)\ge n\bigr)\ge 0.1=\beta.
\]
\end{proof}

\subsection{Information Lower Bound under the Hard Prior}
\label{subsec:info_lower_bound}

We now formalize Step~(B) of the proof and show that the thresholded sampling profile $Y$
must reveal linear information about the hidden set $S^\star$.
Write
\[
X_i:=\ind\{i\in S^\star\}, \qquad X=(X_1,\dots,X_K).
\]
The argument combines the false-negative control from Step~(A) with a regret-based false-positive bound,
and then applies a per-coordinate Fano-type inequality to obtain
\[
I(S^\star;Y)=\Omega(K).
\]

The proof packages Step~(A) and the regret bound into a coordinatewise classification statement:
the thresholded bit $Y_i$ predicts the hidden membership bit $X_i=\ind\{i\in S^\star\}$ with average error bounded away from $1/2$.
A chain-rule / per-coordinate entropy argument then turns this coordinatewise signal into a linear mutual-information lower bound.

\begin{theorem}[Information acquisition lower bound]
\label{thm:info_lower_bound_rigorous}
In the setting of Definition~\ref{def:hard_family_profile}, assume that
algorithm $\mathcal A$ satisfies the regret bound \eqref{eq:reg-def} for some
constant $C>0$, let $c_\Delta,\alpha,\beta>0$ be the constants from
Lemma~\ref{lem:per_good_arm_exploration}, and define
\[
n:=\left\lfloor \alpha \frac{T}{C^2K\log^2(T)\log^2(K)}\right\rfloor.
\]
Assume
\begin{equation}
\label{eq:fp_assumption}
\frac{C\sqrt{KT}\log(T)\log(K)}{n}\le \frac{\beta K}{8},
\end{equation}
and the standing requirements of Lemma~\ref{lem:per_good_arm_exploration}.
Then
\[
I(S^\star;Y)
\ \ge\
\Bigl(1-H_b\bigl(\tfrac12-\tfrac{\beta}{4}\bigr)\Bigr)K
-\frac12\log_2(2K),
\]
where $H_b(p):=-p\log_2 p-(1-p)\log_2(1-p)$ is the binary entropy.
In particular, for all sufficiently large $K$, there exists a universal constant $\gamma>0$ such that
\[
I(S^\star;Y)\ge \gamma K.
\]
\end{theorem}

\begin{remark}[On the horizon regime]
\label{rem:horizon_regime_lower_bound}
For a sufficiently large absolute constant $c$, the condition
\[
T \ge c\, C^6 K\log^6 T\log^6 K
\]
implies both \eqref{eq:fp_assumption} and the standing requirements of
Lemma~\ref{lem:per_good_arm_exploration} (namely $\Delta_0\le 1/4$ and $1\le n\le T/2$).
Hence Theorem~\ref{thm:info_lower_bound_rigorous} applies throughout this regime.
We do not claim the exponent $6$ is intrinsic.

A large-horizon assumption is unavoidable for any lower bound of this form under the target regret rate
$O(\sqrt{KT}\log T\log K)$:
the one-batch algorithm that pulls a fixed arm throughout has regret at most $T$,
and therefore already meets the minimax regret whenever
\[
T \lesssim K\log^2 T\log^2 K.
\]
Thus no statement of the form $B=\Omega(K/W)$ can hold uniformly over all $T$.

The stronger $\log^6$ condition comes from our present proof template.
Specifically, the necessary-exploration step uses the threshold
\[
n \asymp \frac{T}{C^2K\log^2 T\log^2 K},
\]
while the information argument requires the false-positive contribution
\[
\E[\mathrm{FP}] \lesssim \frac{\sqrt{KT}\log T\log K}{n}
\]
to be at most order $K$.
Combining the two yields the displayed sufficient condition.
Optimizing these polylogarithmic exponents is left open.
\end{remark}
\begin{proof}
Define the membership bits $X_i:=\ind\{i\in S^\star\}$ and $X=(X_1,\dots,X_K)$.
Since $S^\star$ and $X$ are in bijection, $H(S^\star\mid Y)=H(X\mid Y)$ and $H(S^\star)=H(X)$.

\paragraph{Bound $H(S^\star)$.}
Since $S^\star$ is uniform over $\binom{K}{K/2}$ subsets,
\[
H(S^\star)=\log_2 \binom{K}{K/2}.
\]
Using $\binom{K}{K/2}\ge \frac{2^K}{\sqrt{2K}}$ gives
\begin{equation}
\label{eq:prior_entropy_lb}
H(S^\star)\ge K-\frac12\log_2(2K).
\end{equation}

\paragraph{Bound $H(S^\star|Y)$}
By the chain rule and the fact that conditioning reduces entropy,
\[
H(X\mid Y)
= \sum_{i=1}^K H(X_i \mid X_{1:i-1},Y)
\le \sum_{i=1}^K H(X_i\mid Y)
\le \sum_{i=1}^K H(X_i\mid Y_i).
\]
For each $i$, let
\[
p_i^\star := \inf_{\psi:\{0,1\}\to\{0,1\}} \P\big(\psi(Y_i)\neq X_i\big)
\]
be the Bayes minimum error probability when predicting $X_i$ from $Y_i$ alone.
The (binary) Fano inequality yields
\[
H(X_i\mid Y_i)\le H_b(p_i^\star).
\]
By concavity of $H_b$,
\[
H(X\mid Y)\le \sum_{i=1}^K H_b(p_i^\star)
\le K H_b\!\Big(\frac1K\sum_{i=1}^K p_i^\star\Big).
\]

\paragraph{1. Upper bound the average Bayes error by the thresholding rule.}
Consider the particular (possibly suboptimal) estimator $\hat X_i:=Y_i$.
Then $p_i^\star \le \P(\hat X_i\neq X_i)=\P(Y_i\neq X_i)$, hence
\[
\frac1K\sum_{i=1}^K p_i^\star
\ \le\ \frac1K\sum_{i=1}^K \P(Y_i\neq X_i)
\ =:\ P_e,
\]
where
\[
KP_e=\E\Big[\sum_{i=1}^K \ind\{Y_i\neq X_i\}\Big]
=\E[\mathrm{FP}]+\E[\mathrm{FN}],
\]
with
\[
\mathrm{FP}:=\sum_{i\notin S^\star}\ind\{Y_i=1\},\qquad
\mathrm{FN}:=\sum_{i\in S^\star}\ind\{Y_i=0\}.
\]

\paragraph{2. Bound FP by regret.}
Fix $S\in\binom{[K]}{K/2}$ and work under $\nu_S$.
By Markov's inequality,
\[
\E_{\nu_S}[\mathrm{FP}]
=\sum_{i\notin S}\P_{\nu_S}(N_i(T)\ge n)
\le \frac1n\sum_{i\notin S}\E_{\nu_S}[N_i(T)].
\]
Under $\nu_S$, the optimal mean is $\mu^\star=1/2$ and any arm $i\notin S$ has mean $0$,
so pulling $i\notin S$ incurs instantaneous expected regret $1/2$, and pulling $i\in S$
incurs $0$.
Thus
\[
\E_{\nu_S}[\Reg_T]=\frac12\sum_{i\notin S}\E_{\nu_S}[N_i(T)],
\qquad\text{so}\qquad
\sum_{i\notin S}\E_{\nu_S}[N_i(T)]=2\E_{\nu_S}[\Reg_T].
\]
Using $\E_{\nu_S}[\Reg_T]\le C\sqrt{KT}\log(T)\log(K)$ gives
\[
\E_{\nu_S}[\mathrm{FP}] \le \frac{2C\sqrt{KT}\log(T)\log(K)}{n}.
\]
Since this bound holds uniformly for all $S\in\binom{[K]}{K/2}$, we can average over the random choice of $S^\star$:
\begin{equation}
\label{eq:FP_bound}
\E[\mathrm{FP}]
= \E_{S^\star}\big[\E[\mathrm{FP} \mid S^\star]\big]
\le \frac{2C\sqrt{KT}\log(T)\log(K)}{n}
\le \frac{\beta K}{4},
\end{equation}
where the last inequality follows from assumption \eqref{eq:fp_assumption}.

\paragraph{3. Bound FN by Lemma~\ref{lem:per_good_arm_exploration}.}
Fix $S\in\binom{[K]}{K/2}$ and work under $\nu_S$.
By Lemma~\ref{lem:per_good_arm_exploration}, for every $j\in S$,
$\P_{\nu_S}(Y_j=1)=\P_{\nu_S}(N_j(T)\ge n)\ge \beta$.
Hence the expected number of true positives satisfies
\[
\E_{\nu_S}\Big[\sum_{j\in S} Y_j\Big]
=\sum_{j\in S}\P_{\nu_S}(Y_j=1)
\ge \beta |S|=\frac{\beta K}{2}.
\]
Therefore
\[
\E_{\nu_S}[\mathrm{FN}]
=|S|-\E_{\nu_S}\Big[\sum_{j\in S}Y_j\Big]
\le \frac{K}{2}\,(1-\beta).
\]
Since this bound holds for any realization $S \in \binom{[K]}{K/2}$, we can obtain the unconditional expectation by marginalizing over $S^\star$:
\begin{equation}
\label{eq:FN_bound}
\E[\mathrm{FN}]
= \E_{S^\star}\big[ \E[\mathrm{FN} \mid S^\star] \big]
\le \frac{K}{2}(1-\beta).
\end{equation}

\paragraph{4. Conclude an upper bound on $H(S^\star\mid Y)$.}
Combining \eqref{eq:FP_bound} and \eqref{eq:FN_bound},
\[
P_e
=\frac{\E[\mathrm{FP}]+\E[\mathrm{FN}]}{K}
\le \frac{\beta}{4}+\frac{1-\beta}{2}
=\frac12-\frac{\beta}{4}
=:\delta < \frac12.
\]
Since $H_b$ is increasing on $[0,1/2]$,
\[
H(S^\star\mid Y)=H(X\mid Y)
\le K H_b\!\Big(\frac1K\sum_{i=1}^K p_i^\star\Big)
\le K H_b(P_e)
\le K H_b(\delta).
\]

\paragraph{Mutual information lower bound.}
Finally,
\[
I(S^\star;Y)=H(S^\star)-H(S^\star\mid Y)
\ge \Big(K-\frac12\log_2(2K)\Big) - K H_b(\delta),
\]
using \eqref{eq:prior_entropy_lb}.
This is exactly the claimed bound with $\delta=\frac12-\frac{\beta}{4}$.
The linear-in-$K$ conclusion follows because $1-H_b(\delta)>0$ for any $\delta\neq 1/2$.
\end{proof}

\paragraph{What remains for the lower bound.}
Theorem~\ref{thm:info_lower_bound_rigorous} completes the instance-level part of the argument:
under the regret guarantee \eqref{eq:reg-def}, even the coarse profile $Y$ must encode $\Omega(K)$ bits about the hidden set $S^\star$.
What remains is now purely structural.
Under batching and a $W$-bit persistent memory budget, how much instance-dependent information can $Y$ carry at all?
The next subsection answers this by showing that $Y$ is measurable with respect to the seed and the boundary-memory transcript, and therefore cannot contain more than $O(BW)$ bits of information.
\subsection{From Information to Batch Complexity}
\label{subsec:batch_complexity}

We now complete the lower bound by upper-bounding how much information the learner can encode in the thresholded sampling profile $Y$ under the batching and memory constraints.
Step~(B) showed that near-minimax regret forces
\[
I(S^\star;Y)=\Omega(K).
\]
To convert this into a lower bound on the number of batches, it remains to show that a $B$-batch algorithm with a $W$-bit persistent state can make $Y$ depend on at most $O(BW)$ bits of instance-dependent information.

The key structural fact is policy commitment: once a batch begins, the complete within-batch action plan is fixed using only the public time, the boundary memory state, and the seed $U$.
Hence rewards observed inside a batch can affect future action plans only through what is retained at later batch boundaries.
This yields the sharper entropy bound $(B-1)W$, and in particular the coarse bound $BW$ used in the roadmap.

The next lemma formalizes the structural consequence we need: although $Y$ is defined from the full pull counts $N_i(T)$, under the batching constraint it is in fact determined by the seed together with the internal boundary memory transcript.

\begin{lemma}[Boundary transcript determines the sampling profile]
\label{lem:Y_function_of_U_M}
Let $\mathcal A$ be a $(B,W)$-batched space-bounded algorithm as in
Section~\ref{subsec:memory_model}, and let $Y$ be the thresholded sampling profile from Definition~\ref{def:hard_family_profile}.
Then $Y$ is a deterministic measurable function of $(U,\mathbf M)$.
\end{lemma}

\begin{proof}[Proof Sketch]
The key point is policy commitment.
For each batch $b$, once $(t_{b-1},M_{t_{b-1}},U)$ is fixed, the entire within-batch action sequence is fixed as well.
Hence the pull-count vector contributed by batch $b$ is determined by the boundary information available when that batch starts.
Summing these batch-level pull counts shows that the full pull-count vector $(N_1(T),\dots,N_K(T))$, and therefore the thresholded profile $Y$, is determined by $(U,\tau,\mathbf M)$, where $\tau=(t_1,\dots,t_{B-1})$ is the realized internal grid.

In the static-grid model, $\tau$ is fixed in advance.
In the adaptive-grid model, each new boundary is chosen from the same boundary information, so $\tau$ is itself determined recursively by $(U,\mathbf M)$.
Therefore $Y$ is a deterministic function of $(U,\mathbf M)$.
The formal measurability verification is deferred to Appendix~\ref{app:boundary_transcript_measurability}.
\end{proof}

\paragraph{Proof of Theorem~\ref{thm:batch_lower_bound_rigorous}}
We are now in position to complete the proof.
Step~(B) lower-bounds $I(S^\star;Y)$ by $\Omega(K)$, while Step~(C) upper-bounds the same quantity by the batch-memory capacity.
Combining the two yields the claimed lower bound on $B$.
\begin{proof}
By Remark~\ref{rem:horizon_regime_lower_bound}, for a sufficiently large universal constant $c$,
the assumption
\[
T \ge c\, C^6 K\log^6 T\log^6 K
\]
implies the hypotheses of Theorem~\ref{thm:info_lower_bound_rigorous} for the explicit admissible choice
\[
c_\Delta=8,\qquad \alpha=\frac{\ln(1.8)}{512},\qquad \beta=0.1.
\]
Hence Theorem~\ref{thm:info_lower_bound_rigorous} yields
\[
I(S^\star;Y)
\ \ge\
\Bigl(1-H_b\bigl(\tfrac12-\tfrac{0.1}{4}\bigr)\Bigr)K
-\frac12\log_2(2K)
=
\gamma_0 K-\frac12\log_2(2K).
\]

By Lemma~\ref{lem:Y_function_of_U_M}, there exists a deterministic and measurable function $g$ such that
\[
Y=g(U,\mathbf M)
\qquad\text{a.s.}
\]
Hence, conditional on $U$, we have the Markov chain
\[
S^\star \ \longrightarrow\ \mathbf M \ \longrightarrow\ Y.
\]
By the conditional data processing inequality,
\[
I(S^\star;Y\mid U)\le I(S^\star;\mathbf M\mid U).
\]
Since $U$ is independent of $S^\star$,
\[
I(S^\star;Y)\le I(S^\star;Y,U)=I(S^\star;Y\mid U).
\]
Therefore
\[
I(S^\star;Y)\le I(S^\star;\mathbf M\mid U)\le H(\mathbf M\mid U)\le H(\mathbf M).
\]

Using the chain rule and $|\mathcal M|\le 2^W$,
\[
H(\mathbf M)
:=\sum_{b=1}^{B-1} H\!\bigl(M_{t_b}\mid M_{t_1},\dots,M_{t_{b-1}}\bigr)
\le \sum_{b=1}^{B-1} H(M_{t_b})
\le \sum_{b=1}^{B-1} \log_2 |\mathcal M|
\le (B-1)W
\le BW.
\]
Thus
\[
I(S^\star;Y)\le BW.
\]

Combining the lower and upper bounds on $I(S^\star;Y)$ gives
\[
\gamma_0 K-\frac12\log_2(2K)\le BW,
\]
which implies
\[
B
\ \ge\
\frac{\gamma_0 K - \frac12\log_2(2K)}{W}.
\]

Since $\gamma_0>0$ is universal, there exists a universal constant $\gamma>0$ such that
for all sufficiently large $K$,
\[
\gamma_0 K-\frac12\log_2(2K)\ge \gamma K.
\]
Hence
\[
B\ge \frac{\gamma K}{W},
\]
which proves the asymptotic lower bound.
The specialization to $W=O(\log T)$ is immediate.
\end{proof}

\paragraph{Interpretation.}
Theorem~\ref{thm:batch_lower_bound_rigorous} formalizes an adaptivity--memory tradeoff:
achieving minimax regret requires learning $\Omega(K)$ bits about which arms are good,
while a $B$-batch algorithm with $W$-bit persistent memory can transport at most $BW$ bits of instance-dependent information overall across batches.
Therefore $BW$ must be at least on the order of $K$.
In particular, if $W=O(\log T)$, then any such algorithm needs
$B=\Omega(K/\log T)$ batches, i.e., a nearly linear number of adaptivity rounds in $K$
when memory is logarithmic.

\section{Upper Bound under Simultaneous Space and Adaptivity Constraints}
\label{sec:upper_bound}

In this section we present a \emph{batched} and \emph{space-bounded} algorithm for the
stochastic $K$-armed bandit problem with horizon $T$ and rewards in $[0,1]$.
Throughout the paper, we use $W$ to denote the persistent-memory budget measured in bits.
To avoid overloading this notation, in the upper-bound construction we use an integer block-size parameter $S\in [K]$.
The algorithm uses $O(S\log T)$ bits of memory, at most $\widetilde{O}(K/S)$ batches, and achieves expected regret $\widetilde{O}(\sqrt{KT})$.

Equivalently, for any bit budget $W$ with $\Omega(\log T) \le W \le O(K\log T)$, choosing $S\asymp W/\log T$ yields an implementation using at most $W$ bits and $\widetilde{O}(K/W)$ batches.
Here the $\Omega(\log T) \le W$ lower bound is needed to represent the relevant counters and empirical means to inverse-polynomial precision in $T$, while the $W \le O(K\log T)$ upper bound is without loss of generality since standard finite-armed bandit algorithms only need to store per-arm summaries for the $K$ arms.
This nearly matches the lower bound of Section~\ref{sec:lower_bound} up to logarithmic factors.

Our algorithm works under the \emph{static-grid} batched model.
Since the static-grid model is algorithmically weaker than the \emph{adaptive-grid} model, the same upper bound immediately applies to the adaptive-grid setting as well.

\subsection{Algorithm Description}

At a high level, our algorithm may be viewed as a space-bounded analogue of a standard batched elimination algorithm \citep{perchet2016batched, gao2019batched,jin2021anytimebatched}.
The main difficulty is that standard batched elimination algorithms keep a large active set of candidate arms across stages, which is impossible under a strict memory budget.

Our solution is to keep only one incumbent together with one block of at most $S$ challengers at a time.
Instead of storing all potentially good arms simultaneously, the algorithm scans the remaining arms block by block.
Within each block, challengers are compared against a frozen benchmark incumbent through a multistage test: if a challenger is confidently worse, it is deactivated for the rest of the block; if it survives all stages and appears better, it may trigger an incumbent update at the end of the block.
Because only one challenger block is kept in memory at a time, this design uses $O(S\log T)$ bits while reducing the number of batches by a factor of $S$ relative to the fully sequential scan.

Concretely, the algorithm has three nested loops.
The outer loop uses a deterministic grid of lengths $t_1,\dots,t_L$ defined by
\[
t_0 \triangleq 1,
\qquad
t_i \triangleq \left\lceil \sqrt{t_{i-1}\cdot \frac{T}{10K}} \right\rceil
\qquad (i\ge 1),
\]
and
\[
L \triangleq \left\lfloor \log_2\log_2\!\left(\frac{T}{10K}\right)\right\rfloor.
\]
Readers familiar with batched bandit literature will recognize that this is the standard ``minimax''-type grid~\citep{perchet2016batched, gao2019batched}.

In each outer iteration $i$, the algorithm maintains a (time-varying) incumbent arm $a_i^\star$ and its empirical mean $\widehat{\mu}_i^\star$.
At the beginning of the iteration, it inherits the final incumbent from the previous outer iteration and denotes this fixed initial identity by
\[
a_{i,\mathrm{init}} \triangleq a_{i-1}^\star.
\]
The algorithm sets the current incumbent variable to $a_i^\star \leftarrow a_{i,\mathrm{init}}$, pulls this arm for $t_i$ rounds, and stores the resulting empirical mean as $\widehat{\mu}_{i}^{\mathrm{init}}$.
It then initializes the incumbent summary by
\[
\widehat{\mu}_i^\star \leftarrow \widehat{\mu}_{i}^{\mathrm{init}}.
\]
Throughout outer iteration $i$, the variable $a_i^\star$ denotes the current incumbent, and $\widehat{\mu}_i^\star$ denotes the stored summary of the arm currently held in $a_i^\star$.
The identity $a_{i,\mathrm{init}}$ stays fixed while the challenger blocks are scanned, even though $a_i^\star$ may later be replaced.

Next, the algorithm partitions
\[
[K]\setminus\{a_{i,\mathrm{init}}\}
\]
into
\[
J \triangleq \left\lceil \frac{K-1}{S}\right\rceil
\]
ordered blocks
\[
C_{i,1},\dots,C_{i,J},
\qquad |C_{i,j}|\le S.
\]
For each block $C_{i,j}$, the algorithm freezes the current incumbent as the benchmark
\[
\bar a_{i,j}\triangleq a_i^\star,
\qquad
\bar \mu_{i,j}\triangleq \widehat{\mu}_i^\star.
\]
It then preallocates $i$ comparison slots indexed by $k=1,\dots,i$, where slot $k$ contains $t_k$ pulls for each arm-position in the block.
If arm $a\in C_{i,j}$ is still active, then slot $k$ is used to pull $a$ and compute $\widehat{\mu}_{i,j,k}(a)$ from $t_k$ fresh samples.
Once $a$ is deactivated, its reserved pulls in the remaining slots are still executed but are reassigned to the frozen benchmark arm $\bar a_{i,j}$.
Thus, for each pair $(i,j)$, the slots are fixed in advance; adaptivity changes only which challengers remain active inside the block, not the batch boundaries.

We use the confidence radius
\[
c(s) \triangleq \sqrt{\frac{\log(2/\delta)}{2s}}
\qquad (s\ge 1).
\]
If, at level $k$, the challenger estimate satisfies
\[
\widehat{\mu}_{i,j,k}(a) < \bar \mu_{i,j} - 2c(t_k),
\]
we mark $a$ as inactive and fill all later slots reserved for this arm in block $(i,j)$ with the frozen benchmark arm.
Crucially, for a fixed slot $k$, the active/inactive status of every arm-position is already determined at the \emph{start} of that slot from the outcomes of the previous slots $1,\dots,k-1$.
Thus the loop over $a\in C_{i,j}$ inside slot $k$ does not require any within-slot feedback: all pulls assigned to slot $k$ can be scheduled simultaneously in one batch, and only after that batch is completed do we update the active set for slot $k+1$.
After level $k=i$, let $\tilde a_{i,j}$ be an active arm in $C_{i,j}$ with the largest value of $\widehat{\mu}_{i,j,i}(a)$, if such an arm exists.
If $\tilde a_{i,j}$ exists and moreover satisfies
\[
\widehat{\mu}_{i,j,i}(\tilde a_{i,j})\ge \bar \mu_{i,j} + 2c(t_i),
\]
then $\tilde a_{i,j}$ replaces the incumbent and we update the stored incumbent summary to $\widehat{\mu}_{i,j,i}(\tilde a_{i,j})$.
After the main outer-loop schedule $i=1,\dots,L$ is completed, all remaining rounds are assigned to the final incumbent $a_L^\star$ in one final batch.

Our algorithm is summarized in Algorithm~\ref{alg:finite_arm_upper}.

\begin{algorithm}
\caption{Simultaneously Batched and Space Constrained Algorithm for standard Multi-Armed Bandit with block size $S$}
\label{alg:finite_arm_upper}
\begin{algorithmic}[1]
\STATE Define $c(s) = \sqrt{\frac{\log(2/\delta)}{2s}}$
\STATE Initialize $a_0^\star \leftarrow 1$ and $t_0 \leftarrow 1$
\STATE Define $t_i \leftarrow \left\lceil \sqrt{t_{i-1}\frac{T}{10K}} \right\rceil$ for $i\ge 1$
\STATE Set $L \leftarrow \left\lfloor \log_2\log_2\!\left(\frac{T}{10K}\right)\right\rfloor$
\FOR{$i = 1$ to $L$}
    \STATE Set $a_{i,\mathrm{init}} \leftarrow a_{i-1}^\star$
    \STATE Set $a_i^\star \leftarrow a_{i,\mathrm{init}}$
    \STATE Pull arm $a_{i,\mathrm{init}}$ for $t_i$ rounds and store its empirical mean $\widehat{\mu}_{i}^{\mathrm{init}}$
    \STATE Set $\widehat{\mu}_i^\star \leftarrow \widehat{\mu}_{i}^{\mathrm{init}}$
    \STATE Set $J \leftarrow \left\lceil \frac{K-1}{S}\right\rceil$ and partition $[K]\setminus\{a_{i,\mathrm{init}}\}$ into ordered blocks $C_{i,1},\dots,C_{i,J}$ of size at most $S$
    \FOR{$j = 1$ to $J$}
        \STATE Set $\bar a_{i,j} \leftarrow a_i^\star$ and $\bar \mu_{i,j} \leftarrow \widehat{\mu}_i^\star$
        \FOR{each $a\in C_{i,j}$}
            \STATE Set $\mathrm{active}(a)\leftarrow \textbf{true}$
        \ENDFOR
        \FOR{$k = 1$ to $i$}
            \FOR{each $a\in C_{i,j}$}
                \IF{$\mathrm{active}(a)$}
                    \STATE Pull arm $a$ for $t_k$ rounds and compute its empirical mean $\widehat{\mu}_{i,j,k}(a)$
                    \IF{$\widehat{\mu}_{i,j,k}(a) < \bar \mu_{i,j} - 2c(t_k)$}
                        \STATE Set $\mathrm{active}(a) \leftarrow \textbf{false}$
                    \ENDIF
                \ELSE
                    \STATE Pull arm $\bar a_{i,j}$ for $t_k$ rounds
                \ENDIF
            \ENDFOR
        \ENDFOR
        \IF{there exists some active arm in $C_{i,j}$}
            \STATE Let $\tilde a_{i,j} \in \arg\max_{a\in C_{i,j}:\,\mathrm{active}(a)} \widehat{\mu}_{i,j,i}(a)$
            \IF{$\widehat{\mu}_{i,j,i}(\tilde a_{i,j}) \ge \bar \mu_{i,j} + 2c(t_i)$}
                \STATE Set $a_i^\star \leftarrow \tilde a_{i,j}$
                \STATE Set $\widehat{\mu}_i^\star \leftarrow \widehat{\mu}_{i,j,i}(\tilde a_{i,j})$
            \ENDIF
        \ENDIF
    \ENDFOR
\ENDFOR
\STATE Pull arm $a_L^\star$ for the final exploitation batch
\end{algorithmic}
\end{algorithm}

\begin{remark}
An outer iteration is \emph{not} a single batch.
Each incumbent-refresh block and each preallocated block-level comparison slot indexed by $(i,j,k)$ is a batch.
Because the schedule is deterministic, the number of rounds used before the final exploitation batch is also deterministic, so the length of the last batch is known in advance.
\end{remark}

\subsection{Algorithm Analysis}

We have the following theorem.
\begin{reptheorem}[Formal]{thm:informal-upper}
    \label{thm:finite_arm_upper}
    There exist universal constants $c,C_1,C_2,C_3>0$ such that the following holds.
    For every integer $K\ge 2$, every block size $S\in [K]$, and every horizon $T\ge cK$, if Algorithm~\ref{alg:finite_arm_upper} is run with $\delta = T^{-4}$, then it defines a static-grid stochastic bandit algorithm that uses no more than $C_1 S \log T$ bits of persistent memory, no more than $C_2 \frac{K}{S}\log^2\log T$ batches, and satisfies
    \[
    \sup_{\nu}\E_\nu[\Reg_T]
    \le C_3\sqrt{KT\log T}\,\log^2\log T,
    \]
    where the supremum is over all $K$-armed stochastic bandit instances with rewards in $[0,1]$.
\end{reptheorem}

Intuitively, the algorithm stores only one incumbent, one block of at most $S$ challengers, and a constant number of counters and empirical means per challenger in the current block, so the memory usage is $O(S\log T)$ bits.
The outer loop $i=1,\dots,L$ forms the main outer-loop schedule, the middle loop scans arm-blocks sequentially to avoid storing a large active set, and the inner loop allocates the prearranged comparison slots for the current block.
The sequential block scan contributes the factor $K/S$ in the batch complexity, while the multistage comparison contributes the additional $\log\log T$ factor.
Hence, for any bit budget $W$ with $\Omega(\log T) \le W \le O(K\log T)$, choosing $S\asymp W/\log T$ yields an implementation using at most $W$ bits and $\widetilde{O}(K/W)$ batches.

The total number of rounds used before the final exploitation batch is deterministic.
Indeed, once $t_1,\dots,t_L$ are fixed, the total number of rounds used before the final exploitation batch is
\begin{equation}
\label{eq:finite_arm_main_schedule_rounds}
N_{\mathrm{main}}
\triangleq
\sum_{i=1}^{L}
\left(
t_i + (K-1)\sum_{k=1}^{i} t_k
\right),
\end{equation}
because outer iteration $i$ contains one incumbent-refresh block of length $t_i$ and then scans the remaining $K-1$ arm-positions through challenger blocks, each position receiving $i$ preallocated slots of lengths $t_1,\dots,t_i$.
Hence the final batch length is exactly $T-N_{\mathrm{main}}$ and is known in advance.

We now prove the theorem in four steps: space, schedule and batch accounting, concentration, and regret.

\begin{lemma}[Space]
\label{lem:finite_arm_space}
Algorithm~\ref{alg:finite_arm_upper} uses only $O(S\log T)$ bits of memory.
\end{lemma}

\begin{proof}
At any moment, the algorithm maintains the current outer index $i$, the current block index $j$, the current comparison level $k$, the current scheduled length $t_i$, the fixed inherited initial arm $a_{i,\mathrm{init}}$, the current incumbent arm $a_i^\star$, one incumbent summary $\widehat{\mu}_i^\star$, the frozen benchmark pair $(\bar a_{i,j},\bar\mu_{i,j})$ for the current block, the identities of the at most $S$ arms in the current block, one active/inactive flag for each of these arms, and at most $S$ empirical summaries $\widehat{\mu}_{i,j,k}(a)$ for the arms currently processed in the block.

All indices and counters are integers bounded by $T$, and therefore each requires only $O(\log T)$ bits.
Each arm identity needs $O(\log K)\le O(\log T)$ bits, so the current block needs $O(S\log T)$ bits.
Each empirical summary is an average of at most $T$ rewards in $[0,1]$, and it suffices to store it to inverse-polynomial precision in $T$, which also requires $O(\log T)$ bits.
Since there are only $O(S)$ such summaries and flags alive simultaneously, the total persistent memory usage is $O(S\log T)$.
\end{proof}

\begin{lemma}[Outer-loop schedule and round budget]
\label{lem:finite_arm_schedule}
Let $N_{\mathrm{main}}$ be defined by~\eqref{eq:finite_arm_main_schedule_rounds}.
Let
\[
L \triangleq \left\lfloor \log_2\log_2\!\left(\frac{T}{10K}\right)\right\rfloor.
\]
The deterministic outer-loop schedule uses a constant fraction of the horizon:
\[
\frac{T}{40}\le N_{\mathrm{main}}\le \frac{3T}{5}.
\]
In particular, the final exploitation batch has deterministic length
\[
\frac{2T}{5}\le T-N_{\mathrm{main}}\le \frac{39T}{40}.
\]
\end{lemma}

\begin{proof}
Let
\[
H \triangleq \frac{T}{10K},
\qquad
s_i \triangleq H^{1-2^{-i}}
\qquad (i\ge 0).
\]
We first compare the integer schedule $(t_i)$ with the smooth proxy $(s_i)$.
We claim that for every $1\le i\le L$,
\begin{equation}
\label{eq:finite_arm_schedule_compare}
s_i \le t_i \le 2s_i.
\end{equation}
For the lower bound, use induction:
\[
t_i
=
\left\lceil \sqrt{Ht_{i-1}} \right\rceil
\ge
\sqrt{Ht_{i-1}}
\ge
\sqrt{Hs_{i-1}}
=
s_i.
\]
For the upper bound, the case $i=1$ is
\[
t_1
=
\lceil \sqrt{H}\rceil
\le
\sqrt{H}+1
\le
2\sqrt{H}
=
2s_1.
\]
Now suppose $i\ge 2$ and $t_{i-1}\le 2s_{i-1}$.
Since $L\ge 1$ implies $H\ge 4$, we have $s_i\ge s_1=\sqrt{H}\ge 2$, and therefore
\[
t_i
=
\left\lceil \sqrt{Ht_{i-1}} \right\rceil
\le
\sqrt{Ht_{i-1}}+1
\le
\sqrt{2Hs_{i-1}}+1
=
\sqrt{2}\,s_i+1
\le
2s_i.
\]
This proves~\eqref{eq:finite_arm_schedule_compare}.

Next, we estimate the last schedule level.
Since $L=\lfloor \log_2\log_2 H\rfloor$, we have
\[
\frac{1}{\log_2 H}\le 2^{-L}<\frac{2}{\log_2 H}.
\]
Hence
\[
\frac{H}{4}
=
H^{1-\frac{2}{\log_2 H}}
<
H^{1-2^{-L}}
=
s_L
\le
H^{1-\frac{1}{\log_2 H}}
=
\frac{H}{2}.
\]
Combining this with~\eqref{eq:finite_arm_schedule_compare} yields
\begin{equation}
\label{eq:finite_arm_schedule_last_level}
\frac{H}{4}< t_L \le H.
\end{equation}

We now control the partial sums.
If $1\le i\le L-2$, then
\[
2^{-i}\ge 2^{-(L-2)} = 4\cdot 2^{-L} \ge \frac{4}{\log_2 H},
\]
and therefore
\[
t_i \le 2H^{1-\frac{4}{\log_2 H}} = \frac{H}{8}
\qquad (1\le i\le L-2).
\]
Therefore, for every $1\le i\le L-2$,
\[
t_{i+1}
\ge
\sqrt{Ht_i}
\ge
2t_i.
\]
Also, since $t_{L-1}\le 2s_{L-1}<H/2$, we have
\[
t_L \ge \sqrt{Ht_{L-1}} \ge \sqrt{2}\,t_{L-1}.
\]
Consequently,
\[
\sum_{k=1}^{L-1} t_k \le 2t_{L-1},
\qquad
\sum_{k=1}^{L} t_k \le 2t_{L-1}+t_L \le (1+\sqrt{2})t_L,
\]
and thus
\[
\sum_{i=1}^{L}\sum_{k=1}^{i} t_k
\le
2\sum_{i=1}^{L-1} t_i + \sum_{k=1}^{L} t_k
\le
4t_{L-1} + (1+\sqrt{2})t_L
\le
(1+3\sqrt{2})t_L
<
6t_L.
\]
Using $t_i+(K-1)\sum_{k=1}^{i} t_k \le K\sum_{k=1}^{i} t_k$, we obtain
\[
N_{\mathrm{main}}
\le
K\sum_{i=1}^{L}\sum_{k=1}^{i} t_k
<
6Kt_L
\le
6KH
=
\frac{3T}{5}.
\]

For the lower bound, the last outer iteration alone contributes at least
\[
t_L + (K-1)t_L = Kt_L > \frac{KH}{4} = \frac{T}{40},
\]
where we used~\eqref{eq:finite_arm_schedule_last_level}.
Hence $N_{\mathrm{main}}\ge T/40$.
Therefore $N_{\mathrm{main}}\in [T/40,3T/5]$, and the final batch length is the predetermined quantity
\[
T-N_{\mathrm{main}}\in \left[\frac{2T}{5},\frac{39T}{40}\right].
\]
\end{proof}

\begin{lemma}[Number of batches]
\label{lem:finite_arm_batches}
Algorithm~\ref{alg:finite_arm_upper} uses at most $O\!\big(\frac{K}{S}\log^2\log T\big)$ batches.
\end{lemma}

\begin{proof}
In outer iteration $i$, the algorithm contains one incumbent-refresh block and
\[
J=\left\lceil \frac{K-1}{S}\right\rceil
\]
challenger blocks, each contributing at most $i$ preallocated comparison slots.
Thus outer iteration $i$ uses at most $1+Ji$ batches.
Including the final exploitation batch, the total number of batches is at most
\[
\sum_{i=1}^{L}\bigl(1+Ji\bigr)+1 = O\!\left(\frac{K}{S}L^2\right).
\]
Since  $L=O(\log\log T)$, we have 
$O\!\left(\frac{K}{S}L^2\right) = O\!\big(\frac{K}{S}\log^2\log T\big)$.
\end{proof}

It remains to bound the regret. We first state a uniform concentration event for the block-level challenger estimates and the incumbent-refresh estimates.

\begin{lemma}[Uniform concentration]
\label{lem:finite_arm_conc}
For every $(i,j,k,a)$ such that arm $a$ is sampled as a challenger at level $k$.
Let $\delta=T^{-4}$ and let $c(\cdot)$ be as defined in the algorithm description.
Then with probability at least $1-\frac{1}{T^2}$, the following event $E$ holds:
\[
E \triangleq \Big\{\forall i,j,k,\forall a\in C_{i,j}:\ |\widehat{\mu}_{i,j,k}(a)-\mu_a|\le c(t_k)\ \text{ and }\ |\widehat{\mu}^{\mathrm{init}}_i-\mu_{a_{i-1}^\star}|\le c(t_i)\Big\}.
\]
\end{lemma}

\begin{proof}
Each empirical mean is an average of independent $[0,1]$-valued rewards.
By Hoeffding's inequality, for any fixed mean based on $s$ samples we have
$\Pr\big(|\widehat{\mu}-\mu|>c(s)\big)\le \delta$.
The total number of empirical means computed is at most
\[
L+\sum_{i=1}^{L}(K-1)i = O(KL^2)=O\!\big(K\log^2\log T\big),
\]
since in outer iteration $i$, each of the $K-1$ challenger positions can produce at most one empirical mean at each of the $i$ levels.
Hence this total is at most $T^2$ whenever $T\ge cK$ for a sufficiently large universal constant $c$.
Applying a union bound over all computed means yields $\Pr(E)\ge 1-T^2\cdot \delta \ge 1-\frac{1}{T^2}$.
\end{proof}

\begin{lemma}[Accuracy of the current incumbent summary]
\label{lem:finite_arm_current_incumbent_summary}
Fix an outer iteration $i$ and work on the event $E$.
At every moment during outer iteration $i$, the stored summary $\widehat{\mu}_i^\star$ satisfies
\[
\bigl|\widehat{\mu}_i^\star-\mu_{a_i^\star}\bigr|\le c(t_i).
\]
In particular, for every challenger block $C_{i,j}$, its frozen benchmark summary satisfies
\[
\bigl|\bar\mu_{i,j}-\mu_{\bar a_{i,j}}\bigr|\le c(t_i).
\]
\end{lemma}

\begin{proof}
During outer iteration $i$, the quantity $\widehat{\mu}_i^\star$ is initialized to
$\widehat{\mu}^{\mathrm{init}}_i$ and can only be updated at the end of a challenger block, when some active arm $\tilde a_{i,j}$ replaces the incumbent, in which case it is set to $\widehat{\mu}_{i,j,i}(\tilde a_{i,j})$.
Therefore $\widehat{\mu}_i^\star$ is always either the initial incumbent estimate $\widehat{\mu}^{\mathrm{init}}_i$ or one of the block-level challenger summaries $\widehat{\mu}_{i,j,i}(a)$.
Both are covered by the event $E$, and both are based on $t_i$ samples.
Hence $\widehat{\mu}_i^\star$ is always within $c(t_i)$ of the true mean of the current incumbent.
Since each frozen benchmark summary $\bar\mu_{i,j}$ is just a copy of $\widehat{\mu}_i^\star$ taken at the start of block $j$, the same bound holds for $(\bar a_{i,j},\bar\mu_{i,j})$.
\end{proof}

\begin{proof}[Proof of Theorem~\ref{thm:finite_arm_upper}]
The space bound follows from Lemma~\ref{lem:finite_arm_space}, the schedule and round-budget bounds from Lemma~\ref{lem:finite_arm_schedule}, and the batch bound from Lemma~\ref{lem:finite_arm_batches}.
It remains to prove the regret bound.
We condition on the event $E$ from Lemma~\ref{lem:finite_arm_conc}.
Within a fixed outer iteration, $a_i^\star$ denotes the current incumbent at the moment under discussion; at the end of the iteration it is the final incumbent of that iteration.
For each block $C_{i,j}$, let $(\bar a_{i,j},\bar\mu_{i,j})$ denote the benchmark frozen at the start of that block.

\paragraph{The best arm is never deactivated.}
Let $a^\star\in\arg\max_{a\in[K]}\mu_a$ denote a true best arm.
Fix an outer iteration $i$, let $j^\star$ be the block containing $a^\star$ if $a^\star\neq a_{i,\mathrm{init}}$, and suppose the algorithm tests $a^\star$ at some level $k\le i$ within that block.
The deactivation condition would be
\[
\widehat{\mu}_{i,j^\star,k}(a^\star) < \bar\mu_{i,j^\star} - 2c(t_k).
\]
Under $E$, we have
\[
\widehat{\mu}_{i,j^\star,k}(a^\star)\ge \mu^\star-c(t_k),
\qquad
\bar\mu_{i,j^\star} \le \mu_{\bar a_{i,j^\star}} + c(t_i)\le \mu^\star + c(t_i)\le \mu^\star + c(t_k),
\]
where the second inequality uses Lemma~\ref{lem:finite_arm_current_incumbent_summary}, and the last one uses the monotonicity of $c(\cdot)$.
Therefore
\[
\widehat{\mu}_{i,j^\star,k}(a^\star)\ge \bar\mu_{i,j^\star} - 2c(t_k),
\]
so the deactivation condition can never hold for $a^\star$.
Hence $a^\star$ always survives to level $k=i$ whenever it is tested in outer iteration $i$.

\paragraph{Incumbent true means do not decrease across block updates within an outer iteration.}
Fix an outer iteration $i$, a block $C_{i,j}$, and suppose its selected arm $\tilde a_{i,j}$ replaces the current incumbent.
Immediately before the update, by the update rule,
\[
\widehat{\mu}_{i,j,i}(\tilde a_{i,j}) \ge \bar\mu_{i,j} + 2c(t_i).
\]
Under $E$, we have
\[
\mu_{\tilde a_{i,j}} \ge \widehat{\mu}_{i,j,i}(\tilde a_{i,j})-c(t_i)
\ge \bar\mu_{i,j} + c(t_i)
\ge \mu_{\bar a_{i,j}},
\]
where the last inequality uses Lemma~\ref{lem:finite_arm_current_incumbent_summary}.
Thus every block update can only increase the true mean of the incumbent.
Consequently, every incumbent that appears during outer iteration $i$ has true mean at least that of the initial incumbent, which is $a_{i-1}^\star$.

\paragraph{Final incumbent gap bound.}
Fix an outer iteration $i$.
If $a^\star=a_{i,\mathrm{init}}$, then the incumbent is already optimal at the start of the iteration, and the previous paragraph implies that every later incumbent in this iteration is no worse.
Thus the final incumbent has gap $0$.
Now consider the case $a^\star\neq a_{i,\mathrm{init}}$.
Let $j^\star$ be the block containing $a^\star$.
When the algorithm tests $a^\star$ in block $j^\star$, we showed that the best arm is never deactivated, so it reaches level $k=i$.
At the end of that block, there are two cases.

If the block update succeeds, then the algorithm installs some arm $\tilde a_{i,j^\star}$ satisfying
\[
\widehat{\mu}_{i,j^\star,i}(\tilde a_{i,j^\star})\ge \widehat{\mu}_{i,j^\star,i}(a^\star),
\]
because $\tilde a_{i,j^\star}$ is chosen as the active arm with the largest final empirical mean.
Under $E$,
\[
\mu_{\tilde a_{i,j^\star}}
\ge \widehat{\mu}_{i,j^\star,i}(\tilde a_{i,j^\star})-c(t_i)
\ge \widehat{\mu}_{i,j^\star,i}(a^\star)-c(t_i)
\ge \mu^\star-2c(t_i),
\]
so the updated incumbent has gap at most $2c(t_i)$.

If the block update fails, then
\[
\widehat{\mu}_{i,j^\star,i}(\tilde a_{i,j^\star}) < \bar\mu_{i,j^\star} + 2c(t_i).
\]
Since $a^\star$ is still active at level $i$, maximality of $\tilde a_{i,j^\star}$ implies
\[
\widehat{\mu}_{i,j^\star,i}(a^\star)\le \widehat{\mu}_{i,j^\star,i}(\tilde a_{i,j^\star}) < \bar\mu_{i,j^\star} + 2c(t_i).
\]
Therefore, under $E$,
\[
\mu^\star-\mu_{\bar a_{i,j^\star}}
\le \bigl(\mu^\star-\widehat{\mu}_{i,j^\star,i}(a^\star)\bigr)
   + \bigl(\widehat{\mu}_{i,j^\star,i}(a^\star)-\bar\mu_{i,j^\star}\bigr)
   + \bigl(\bar\mu_{i,j^\star}-\mu_{\bar a_{i,j^\star}}\bigr)
< c(t_i)+2c(t_i)+c(t_i)
= 4c(t_i),
\]
and hence the incumbent present after block $j^\star$ has gap at most $4c(t_i)$.
Moreover, incumbent means do not decrease within the remaining blocks of outer iteration $i$, so the final incumbent at the end of outer iteration $i$ also satisfies
\begin{equation}
\label{eq:finite_arm_inc_gap}
\Delta_{a_i^\star} \triangleq \mu^\star-\mu_{a_i^\star} \le 4c(t_i).
\end{equation}

\paragraph{Survival implies small gap.}
Fix an outer iteration $i$, a block $C_{i,j}$, an arm $a\in C_{i,j}$, and a level $k\ge 2$.
If arm $a$ reaches level $k$ (i.e., it is not deactivated at level $k-1$), then
\[
\widehat{\mu}_{i,j,k-1}(a) \ge \bar\mu_{i,j} - 2c(t_{k-1}).
\]
Under $E$, we have
\[
\mu_{\bar a_{i,j}}-\mu_a \le 4c(t_{k-1}).
\]
where we used
\[
\mu_{\bar a_{i,j}}-\mu_a
\le \bigl(\mu_{\bar a_{i,j}}-\bar\mu_{i,j}\bigr)
   + \bigl(\bar\mu_{i,j}-\widehat{\mu}_{i,j,k-1}(a)\bigr)
   + \bigl(\widehat{\mu}_{i,j,k-1}(a)-\mu_a\bigr)
\le c(t_i)+2c(t_{k-1})+c(t_{k-1})
\le 4c(t_{k-1}),
\]
and the first term is bounded using Lemma~\ref{lem:finite_arm_current_incumbent_summary}.
The initial incumbent of outer iteration $i$ is $a_{i-1}^\star$, and by~\eqref{eq:finite_arm_inc_gap} applied to outer iteration $i-1$,
\[
\Delta_{a_{i-1}^\star} \le 4c(t_{i-1}) \le 4c(t_{k-1}).
\]
Since incumbent means do not decrease across block updates within an outer iteration, the block benchmark arm $\bar a_{i,j}$ cannot be worse than the initial incumbent $a_{i-1}^\star$, hence
\[
\Delta_{\bar a_{i,j}} \le 4c(t_{k-1}).
\]
Therefore,
\begin{equation}
\label{eq:finite_arm_survival_gap}
\Delta_a \le 8c(t_{k-1}).
\end{equation}

\paragraph{Regret decomposition.}
Let $k_i(a)$ denote the largest level at which arm $a$ itself is sampled in outer iteration $i$.
If $k_i(a)<i$, then the remaining slots $k=k_i(a)+1,\dots,i$ reserved for arm $a$ are filled by the frozen benchmark of its block.
This decomposition mirrors the predetermined batch layout.
The regret therefore has four contributions:
(i) incumbent-refresh blocks at the beginning of the outer iterations,
(ii) challenger pulls before deactivation or replacement,
(iii) benchmark-filled tail slots after a challenger is deactivated, and
(iv) the final exploitation batch.

\emph{(i) Incumbent-refresh blocks.}
At the start of outer iteration $i$, the algorithm pulls the current incumbent (which is $a_{i-1}^\star$) for $t_i$ rounds.
The $i=1$ term is at most $t_1$.
For $i\ge 2$, using~\eqref{eq:finite_arm_inc_gap} for iteration $i-1$ yields $\Delta_{a_{i-1}^\star}\le 4c(t_{i-1})$, hence
\[
\sum_{i=1}^{L} t_i\,\Delta_{a_{i-1}^\star}
\le t_1 + \sum_{i=2}^{L} t_i\cdot 4c(t_{i-1})
= O\!\left(L\sqrt{\tfrac{T}{K}\log T}\right),
\]
where we used $\tfrac{t_i}{\sqrt{t_{i-1}}}=\Theta\!\left(\sqrt{\tfrac{T}{10K}}\right)$ from the schedule.

\emph{(ii) Challenger slots before deactivation.}
The challenger pulls of arm $a$ in outer iteration $i$ occur exactly at levels $k=1,\dots,k_i(a)$ inside the unique block that contains $a$.
The $k=1$ terms contribute at most $O(KLt_1)$ and are dominated by the final bound.
Whenever arm $a$ is pulled at some level $k\ge 2$, it must have survived level $k-1$, and thus~\eqref{eq:finite_arm_survival_gap} applies.
Therefore the regret incurred by all challenger pulls at levels $k\ge 2$ is at most
\begin{align*}
\sum_{i=1}^{L}\sum_{a\in[K]}\sum_{k=2}^{k_i(a)} t_k\,\Delta_a
&\le \sum_{i=1}^{L}\sum_{a\in[K]}\sum_{k=2}^{k_i(a)} t_k\cdot 8c(t_{k-1})\\
&\le \sum_{i=1}^{L}\sum_{a\in[K]}\sum_{k=2}^{i} t_k\cdot 8c(t_{k-1})\\
&= O\!\left(\sum_{i=1}^{L}\sum_{a\in[K]}\sum_{k=2}^{i} \frac{t_k}{\sqrt{t_{k-1}}}\cdot \sqrt{\log T}\right)\\
&= O\!\left(L^2\sqrt{KT\log T}\right),
\end{align*}
using again $\tfrac{t_k}{\sqrt{t_{k-1}}}=\Theta\!\left(\sqrt{\tfrac{T}{10K}}\right)$.

\emph{(iii) Benchmark-filled tail slots.}
Suppose challenger $a$ is deactivated during outer iteration $i$ inside block $C_{i,j}$.
Then the remaining slots reserved for this arm are filled by the frozen benchmark arm $\bar a_{i,j}$ of that block.
By the monotonicity established above, every block benchmark appearing during outer iteration $i$ is at least as good as the initial incumbent $a_{i-1}^\star$.
Hence, for $i\ge 2$, each such filler slot incurs regret at most $4c(t_{i-1})$ by~\eqref{eq:finite_arm_inc_gap}.
Therefore the total regret of these benchmark-filled tail slots is at most
\begin{align*}
\sum_{i=2}^{L}\sum_{a\in[K]}\sum_{k=k_i(a)+1}^{i} t_k \cdot \Delta_{\mathrm{inc},i,a,k}
&\le \sum_{i=2}^{L}(K-1)\Bigl(\sum_{k=1}^{i} t_k\Bigr)\cdot 4c(t_{i-1})\\
&\le 4(1+\sqrt{2})(K-1)\sum_{i=2}^{L} t_i c(t_{i-1})\\
&= O\!\left(L\sqrt{KT\log T}\right),
\end{align*}
where $\Delta_{\mathrm{inc},i,a,k}$ denotes the gap of the benchmark arm used to fill that slot.
Thus replacing a deactivated challenger by the block benchmark does not change the order of the regret bound.

\emph{(iv) Final exploitation.}
After the main outer-loop schedule $i=1,\dots,L$ terminates, the algorithm commits to the final incumbent $a_L^\star$ for the remaining
$T-N_{\mathrm{main}}$ rounds.
By Lemma~\ref{lem:finite_arm_schedule}, this final batch length is deterministic, and by the final incumbent gap bound together with $t_L\asymp \tfrac{T}{10K}$, we have
\[
\bigl(T-N_{\mathrm{main}}\bigr)\Delta_{a_L^\star}
\le T\cdot 4c(t_L)
= O\!\left(\sqrt{KT\log T}\right).
\]

Combining (i)--(iv) and using $L=O(\log\log T)$ completes the proof under $E$.
Taking expectation over $E$ and its complement (whose probability is at most $1/T^2$) yields the stated bound.
Therefore the explicit batched formulation attains the same upper bound.
\end{proof}

\section{Conclusion}
\label{sec:conclusion}

We identify a qualitative transition in stochastic bandits under joint constraints on memory and adaptivity.
When either resource is available in abundance, near-minimax regret is attainable at surprisingly low cost:
logarithmic memory suffices under fully sequential interaction, and a $K$-independent number of batches suffices when memory is unrestricted.
Our results show that this benign picture breaks down when the two constraints are imposed simultaneously.
In the large-horizon regime, near-minimax regret requires
\[
B=\Omega(K/W),
\]
even under adaptive grids, and our static-grid upper bound shows that this dependence on $K/W$ is achievable up to polylogarithmic factors.
Thus the batch complexity of near-minimax learning is governed by the memory--adaptivity trade-off captured by $K/W$.

Beyond the specific lower bound, the main conceptual message is that statistical performance in interactive learning is limited not only by how much information is acquired, but also by how efficiently that information can be stored and transported across rounds of adaptivity.
Our proof makes this precise through an information-bottleneck viewpoint:
low regret forces the learner to extract $\Omega(K)$ bits about the hidden good-arm set, while a $B$-batch protocol with a $W$-bit memory budget can preserve only $O(BW)$ bits of instance-dependent information across batch boundaries.
In this sense, the relevant resource trade-off is not merely exploration versus exploitation, but information acquisition versus information retention.

On the technical side, the lower bound combines a coarse sampling profile, a per-coordinate information argument, and an event-restricted change-of-measure lemma tailored to under-sampling events.
The last ingredient may be useful beyond the present proof, as it isolates a clean way to compare adaptive-sampling processes while paying only for the first $n$ pulls relevant to the event under consideration.

Several directions remain open.
The most immediate algorithmic question is to sharpen the remaining polylogarithmic gap in the $K/W$ trade-off.
It would also be interesting to understand whether the same memory--adaptivity bottleneck persists in richer models such as linear, generalized linear, and contextual bandits, where the latent instance information is structured rather than arm-wise independent.







\appendix
\section{Deferred Measurability Proofs}

\subsection{Localized change of measure: formal setup and proofs}
\label{app:localized_com}

This subsection collects the measurable-space formulation and the truncation-based proofs
used in Section~\ref{subsec:necessary_exploration}.

\paragraph{Formal setup.}
Fix a measurable reward space $(\mathcal Y,\mathcal S)$.
Work on a common measurable space $(\Omega,\mathcal F)$ supporting a random seed $U\sim\lambda$ and an array
of reward variables $(X_{i,\ell})_{i\in[K],\,\ell\ge1}$ taking values in $\mathcal Y$.
Let $\P_0:=\P_{\nu_0}$ and $\P_1:=\P_{\nu_1}$ denote two probability measures on $(\Omega,\mathcal F)$ such that:
\begin{itemize}
\item $U$ is independent of all reward variables under both $\P_0$ and $\P_1$;
\item the collection $(X_{i,\ell})_{i\in[K],\,\ell\ge1}$ is mutually independent under both $\P_0$ and $\P_1$;
\item for every $i\neq j$, the stream $(X_{i,\ell})_{\ell\ge1}$ has the same i.i.d.\ law under $\P_0$ and $\P_1$;
\item for arm $j$, the stream $(X_{j,\ell})_{\ell\ge1}$ is i.i.d.\ with law $P$ under $\P_0$ and i.i.d.\ with law $Q$ under $\P_1$.
\end{itemize}
Write $\E_0,\E_1$ for the corresponding expectations.

The algorithm is a non-anticipating policy:
for each $t\ge1$,
\[
A_t=\rho_t\bigl(U,A_1,R_1,\dots,A_{t-1},R_{t-1}\bigr)\in[K],
\]
for some measurable map $\rho_t$.
The observed reward is
\[
R_t:=X_{A_t,\,N_{A_t}(t)},
\qquad
N_i(t):=\sum_{s=1}^t \ind\{A_s=i\}.
\]
Let \((\mathcal F_t)_{t=0}^T\) denote the natural filtration generated by the interaction transcript,
\[
\mathcal F_t=\sigma(U,A_1,R_1,\dots,A_t,R_t).
\]
For each integer $n\ge0$, define the prefix sigma-field
\[
\mathcal G_n
:=
\sigma\!\Bigl(U,\ (X_{i,\ell})_{i\neq j,\ \ell\ge1},\ X_{j,1},\dots,X_{j,n}\Bigr).
\]

\begin{lemma}[Budget-event measurability]
\label{lem:budget_event_measurable}
For every event $E\in\mathcal F_T$ and every integer $n\ge0$,
\[
E\cap\{N_j(T)\le n\}\in\mathcal G_n.
\]
\end{lemma}

\begin{proof}
Fix $n\ge0$ and choose an arbitrary point $y_\star\in\mathcal Y$.
Define truncated reward streams by
\[
\widetilde X_{i,\ell}^{(n)}
:=
\begin{cases}
X_{i,\ell}, & i\neq j,\\
X_{j,\ell}, & i=j,\ \ell\le n,\\
y_\star, & i=j,\ \ell>n.
\end{cases}
\]
Let $\bigl(\widetilde A_t^{(n)},\widetilde R_t^{(n)}\bigr)_{t=1}^T$ be the transcript obtained by running the
same policy on the truncated array, namely
\[
\widetilde A_t^{(n)}
=
\rho_t\!\bigl(U,\widetilde A_1^{(n)},\widetilde R_1^{(n)},\dots,
\widetilde A_{t-1}^{(n)},\widetilde R_{t-1}^{(n)}\bigr),
\]
and
\[
\widetilde R_t^{(n)}
=
\widetilde X^{(n)}_{\widetilde A_t^{(n)},\,\widetilde N_{\widetilde A_t^{(n)}}^{(n)}(t)},
\qquad
\widetilde N_i^{(n)}(t):=\sum_{s=1}^t \ind\{\widetilde A_s^{(n)}=i\}.
\]
By construction, the truncated transcript
\[
\widetilde H_T^{(n)}
:=
\bigl(U,\widetilde A_1^{(n)},\widetilde R_1^{(n)},\dots,\widetilde A_T^{(n)},\widetilde R_T^{(n)}\bigr)
\]
and the count $\widetilde N_j^{(n)}(T)$ are $\mathcal G_n$-measurable.

Now define the first times at which arm $j$ is pulled for the $(n+1)$-st time:
\[
\tau_j^{(n)}:=\inf\{t\ge1:N_j(t)=n+1\},
\qquad
\widetilde\tau_j^{(n)}:=\inf\{t\ge1:\widetilde N_j^{(n)}(t)=n+1\},
\]
with the convention $\inf\emptyset=\infty$.
We claim that $\tau_j^{(n)}=\widetilde\tau_j^{(n)}$.

Indeed, we prove by induction on $t$ that the original and truncated transcripts coincide up to every time
$t<\tau_j^{(n)}\wedge\widetilde\tau_j^{(n)}$.
The base case $t=0$ is trivial.
Assume the claim holds up to time $t-1$ for some $t<\tau_j^{(n)}\wedge\widetilde\tau_j^{(n)}$.
Then the histories up to time $t-1$ are identical, so the actions at time $t$ are identical because the
policy is non-anticipating and uses the same seed $U$.
If the common action is some $i\neq j$, then the observed rewards coincide because the two arrays agree on arm $i$.
If the common action is $j$, then, since $t<\tau_j^{(n)}\wedge\widetilde\tau_j^{(n)}$, both processes still have at most $n$
pulls of arm $j$ after round $t$, so the reward index used at time $t$ is at most $n$ in both processes,
and hence the observed rewards also coincide because the two arrays agree on
$X_{j,1},\dots,X_{j,n}$.
Thus the transcripts coincide up to time $t$.

If $\tau_j^{(n)}<\widetilde\tau_j^{(n)}$, then the transcripts coincide up to time $\tau_j^{(n)}-1$, so the actions at time $\tau_j^{(n)}$
also coincide.
Since $N_j(\tau_j^{(n)})=n+1$, we must have $A_{\tau_j^{(n)}}=j$, hence also $\widetilde A_{\tau_j^{(n)}}^{(n)}=j$, which forces
$\widetilde\tau_j^{(n)}\le\tau_j^{(n)}$, a contradiction.
The case $\widetilde\tau_j^{(n)}<\tau_j^{(n)}$ is identical.
Therefore $\tau_j^{(n)}=\widetilde\tau_j^{(n)}$, and hence
\[
\{N_j(T)\le n\}
=
\{\widetilde N_j^{(n)}(T)\le n\}
\in\mathcal G_n.
\]
Moreover, on this event the original and truncated full transcripts coincide up to time $T$.

Finally, let
\[
H_T := (U,A_1,R_1,\dots,A_T,R_T).
\]
Since $E\in\mathcal F_T=\sigma(H_T)$, there exists a measurable set $\Gamma_E$ in the transcript space such that
\[
E=\{H_T\in \Gamma_E\}.
\]
Using the coincidence of the original and truncated transcripts on $\{N_j(T)\le n\}$, we obtain
\[
E\cap\{N_j(T)\le n\}
=
\{\widetilde H_T^{(n)}\in \Gamma_E\}\cap\{\widetilde N_j^{(n)}(T)\le n\},
\]
and the right-hand side is $\mathcal G_n$-measurable.
\end{proof}

\subsection{Formal proof that the boundary transcript determines the sampling profile}
\label{app:boundary_transcript_measurability}

\begin{proof}[Proof of Lemma~\ref{lem:Y_function_of_U_M}]
For each batch $b\in[B]$ and arm $i\in[K]$, define the batch-level pull count
\[
N_i^{(b)} := \sum_{t=t_{b-1}+1}^{t_b} \ind\{A_t=i\},
\]
and let
\[
N^{(b)} := (N_i^{(b)})_{i\in[K]}.
\]
By the policy-commitment constraint, conditional on $(t_{b-1},M_{t_{b-1}},U)$ the entire action plan for batch $b$ is fixed.
Since $t_{b-1}$ is one coordinate of the internal grid $\tau$, where $\tau:=(t_1,\dots,t_{B-1})$ denotes the realized internal grid, the vector
$N^{(b)}$ is $\sigma(M_{t_{b-1}},U,\tau)$-measurable.
Hence there exists a measurable map $c_b$ such that
\[
N^{(b)} = c_b(M_{t_{b-1}},U,\tau)\qquad\text{a.s.}
\]

Summing over batches yields $N_i(T)=\sum_{b=1}^B N_i^{(b)}$, and therefore
\[
Y_i
:=
\ind\Bigl\{\sum_{b=1}^B c_{b,i}(M_{t_{b-1}},U,\tau)\ge n\Bigr\},
\qquad i\in[K].
\]
Since $M_0$ is fixed, this proves that $Y$ is a deterministic function of
\[
(U,\tau,M_{t_1},\dots,M_{t_{B-1}}).
\]

In the static-grid model, $\tau$ is fixed before interaction and therefore is a deterministic function of $U$ (or simply deterministic if the grid is nonrandom).
In the adaptive-grid model, we have
\[
t_1=\psi_1(0,M_0,U),
\qquad
t_b=\psi_b(t_{b-1},M_{t_{b-1}},U),\quad b=2,\dots,B-1,
\]
so $\tau$ is recursively determined by $(U,M_{t_1},\dots,M_{t_{B-1}})$.
Hence in either model $Y$ is a deterministic function of
\[
(U,M_{t_1},\dots,M_{t_{B-1}}),
\]
as claimed. Even in the adaptive-grid model, the realized grid introduces no additional instance-dependent information beyond the seed and the boundary-memory transcript.
\end{proof}

\bibliographystyle{alpha}
\bibliography{refs}

@article{perchet2016batched,
  title={Batched bandit problems},
  author={Perchet, Vianney and Rigollet, Philippe and Chassang, Sylvain and Snowberg, Erik},
  journal={The Annals of Statistics},
  pages={660--681},
  year={2016},
  publisher={JSTOR}
}

@article{gao2019batched,
  title={Batched multi-armed bandits problem},
  author={Gao, Zijun and Han, Yanjun and Ren, Zhimei and Zhou, Zhengqing},
  journal={Advances in Neural Information Processing Systems},
  volume={32},
  year={2019}
}

@inproceedings{esfandiari2021regret,
  title={Regret bounds for batched bandits},
  author={Esfandiari, Hossein and Karbasi, Amin and Mehrabian, Abbas and Mirrokni, Vahab},
  booktitle={Proceedings of the AAAI Conference on Artificial Intelligence},
  volume={35},
  number={8},
  pages={7340--7348},
  year={2021}
}

@inproceedings{jin2021anytimebatched,
  title={Almost optimal anytime algorithm for batched multi-armed bandits},
  author={Jin, Tianyuan and Tang, Jing and Xu, Pan and Huang, Keke and Xiao, Xiaokui and Gu, Quanquan},
  booktitle={International Conference on Machine Learning},
  pages={5065--5073},
  year={2021},
  organization={PMLR}
}

@article{kalkanli2021batchedthompson,
  title={Batched thompson sampling},
  author={Kalkanli, Cem and Ozgur, Ayfer},
  journal={Advances in Neural Information Processing Systems},
  volume={34},
  pages={29984--29994},
  year={2021}
}

@article{cesabianchi2013switchingcosts,
  title={Online learning with switching costs and other adaptive adversaries},
  author={Cesa-Bianchi, Nicolo and Dekel, Ofer and Shamir, Ohad},
  journal={Advances in Neural Information Processing Systems},
  volume={26},
  year={2013}
}

@inproceedings{joulani2013delayed,
  title={Online learning under delayed feedback},
  author={Joulani, Pooria and Gyorgy, Andras and Szepesv{\'a}ri, Csaba},
  booktitle={International conference on machine learning},
  pages={1453--1461},
  year={2013},
  organization={PMLR}
}

@inproceedings{vernade2020delayedlinear,
  title={Linear bandits with stochastic delayed feedback},
  author={Vernade, Claire and Carpentier, Alexandra and Lattimore, Tor and Zappella, Giovanni and Ermis, Beyza and Brueckner, Michael},
  booktitle={International Conference on Machine Learning},
  pages={9712--9721},
  year={2020},
  organization={PMLR}
}

@inproceedings{agarwal2017limitedadaptivity,
  title={Learning with limited rounds of adaptivity: Coin tossing, multi-armed bandits, and ranking from pairwise comparisons},
  author={Agarwal, Arpit and Agarwal, Shivani and Assadi, Sepehr and Khanna, Sanjeev},
  booktitle={Conference on Learning Theory},
  pages={39--75},
  year={2017},
  organization={PMLR}
}

@inproceedings{jun2016toparmbatch,
  title={Top arm identification in multi-armed bandits with batch arm pulls},
  author={Jun, Kwang-Sung and Jamieson, Kevin and Nowak, Robert and Zhu, Xiaojin},
  booktitle={Artificial Intelligence and Statistics},
  pages={139--148},
  year={2016},
  organization={PMLR}
}

@inproceedings{chaudhuri2020boundedarmmemory,
  title={Regret minimisation in multi-armed bandits using bounded arm memory},
  author={Chaudhuri, Arghya Roy and Kalyanakrishnan, Shivaram},
  booktitle={Proceedings of the AAAI Conference on Artificial Intelligence},
  volume={34},
  number={06},
  pages={10085--10092},
  year={2020}
}

@article{maiti2021boundedarmmemory,
  title={Multi-armed bandits with bounded arm-memory: Near-optimal guarantees for best-arm identification and regret minimization},
  author={Maiti, Arnab and Patil, Vishakha and Khan, Arindam},
  journal={Advances in Neural Information Processing Systems},
  volume={34},
  pages={19553--19565},
  year={2021}
}

@inproceedings{liau2018constantspace,
  title={Stochastic multi-armed bandits in constant space},
  author={Liau, David and Song, Zhao and Price, Eric and Yang, Ger},
  booktitle={International Conference on Artificial Intelligence and Statistics},
  pages={386--394},
  year={2018},
  organization={PMLR}
}

@article{xu2021memoryconstrained,
  title={Memory-constrained no-regret learning in adversarial multi-armed bandits},
  author={Xu, Xiao and Zhao, Qing},
  journal={IEEE Transactions on Signal Processing},
  volume={69},
  pages={2371--2382},
  year={2021},
  publisher={IEEE}
}

@inproceedings{baudry2021limitedmemory,
  title={On limited-memory subsampling strategies for bandits},
  author={Baudry, Dorian and Russac, Yoan and Capp{\'e}, Olivier},
  booktitle={International Conference on Machine Learning},
  pages={727--737},
  year={2021},
  organization={PMLR}
}

@inproceedings{assadi2020exploration,
  title={Exploration with limited memory: streaming algorithms for coin tossing, noisy comparisons, and multi-armed bandits},
  author={Assadi, Sepehr and Wang, Chen},
  booktitle={Proceedings of the 52nd Annual ACM SIGACT Symposium on theory of computing},
  pages={1237--1250},
  year={2020}
}

@inproceedings{jin2021optimalstreaming,
  title={Optimal streaming algorithms for multi-armed bandits},
  author={Jin, Tianyuan and Huang, Keke and Tang, Jing and Xiao, Xiaokui},
  booktitle={International Conference on Machine Learning},
  pages={5045--5054},
  year={2021},
  organization={PMLR}
}

@inproceedings{agarwal2022sharp,
  title={A sharp memory-regret trade-off for multi-pass streaming bandits},
  author={Agarwal, Arpit and Khanna, Sanjeev and Patil, Prathamesh},
  booktitle={Conference on Learning Theory},
  pages={1423--1462},
  year={2022},
  organization={PMLR}
}

@article{li2023tightmemoryregret,
  title={Tight memory-regret lower bounds for streaming bandits},
  author={Li, Shaoang and Zhang, Lan and Wang, Junhao and Li, Xiang-Yang},
  journal={arXiv preprint arXiv:2306.07903},
  year={2023}
}

@article{assadi2022singlepasslowerbounds,
  title={Single-pass streaming lower bounds for multi-armed bandits exploration with instance-sensitive sample complexity},
  author={Assadi, Sepehr and Wang, Chen},
  journal={Advances in Neural Information Processing Systems},
  volume={35},
  pages={33066--33079},
  year={2022}
}

@inproceedings{wang2023singlepass,
  title={Tight regret bounds for single-pass streaming multi-armed bandits},
  author={Wang, Chen},
  booktitle={International Conference on Machine Learning},
  pages={35525--35547},
  year={2023},
  organization={PMLR}
}

@book{lattimore2020banditalgorithms,
  title={Bandit algorithms},
  author={Lattimore, Tor and Szepesv{\'a}ri, Csaba},
  year={2020},
  publisher={Cambridge University Press}
}

@article{auer2002finite,
  title={Finite-time analysis of the multiarmed bandit problem},
  author={Auer, Peter and Cesa-Bianchi, Nicolo and Fischer, Paul},
  journal={Machine learning},
  volume={47},
  number={2},
  pages={235--256},
  year={2002},
  publisher={Springer}
}

@article{thompson1933likelihood,
  title={On the likelihood that one unknown probability exceeds another in view of the evidence of two samples},
  author={Thompson, William R},
  journal={Biometrika},
  volume={25},
  number={3/4},
  pages={285--294},
  year={1933},
  publisher={JSTOR}
}

@inproceedings{agrawal2012analysis,
  title={Analysis of thompson sampling for the multi-armed bandit problem},
  author={Agrawal, Shipra and Goyal, Navin},
  booktitle={Conference on learning theory},
  pages={39--1},
  year={2012},
  organization={JMLR Workshop and Conference Proceedings}
}

@article{cover2003two,
  title={The two-armed-bandit problem with time-invariant finite memory},
  author={Cover, T and Hellman, M},
  journal={IEEE Transactions on Information Theory},
  volume={16},
  number={2},
  pages={185--195},
  year={2003},
  publisher={IEEE}
}

@article{kaufmann2016complexity,
  title={On the complexity of best-arm identification in multi-armed bandit models},
  author={Kaufmann, Emilie and Capp{\'e}, Olivier and Garivier, Aur{\'e}lien},
  journal={The Journal of Machine Learning Research},
  volume={17},
  number={1},
  pages={1--42},
  year={2016},
  publisher={JMLR. org}
}

@inproceedings{garivier2016optimal,
  title={Optimal best arm identification with fixed confidence},
  author={Garivier, Aur{\'e}lien and Kaufmann, Emilie},
  booktitle={Conference on Learning Theory},
  pages={998--1027},
  year={2016},
  organization={PMLR}
}

@article{bretagnolle1979estimation,
  title={Estimation des densit{\'e}s: risque minimax},
  author={Bretagnolle, Jean and Huber, Catherine},
  journal={Zeitschrift f{\"u}r Wahrscheinlichkeitstheorie und verwandte Gebiete},
  volume={47},
  number={2},
  pages={119--137},
  year={1979},
  publisher={Springer}
}

@incollection{tsybakov2008nonparametric,
  title={Nonparametric estimators},
  author={Tsybakov, Alexandre B},
  booktitle={Introduction to Nonparametric Estimation},
  pages={1--76},
  year={2008},
  publisher={Springer}
}

@article{arora2012policyregret,
  title={Online bandit learning against an adaptive adversary: from regret to policy regret},
  author={Arora, Raman and Dekel, Ofer and Tewari, Ambuj},
  journal={arXiv preprint arXiv:1206.6400},
  year={2012}
}

@article{cover1968note,
  title={A note on the two-armed bandit problem with finite memory},
  author={Cover, Thomas M},
  journal={Information and Control},
  volume={12},
  number={5},
  pages={371--377},
  year={1968},
  publisher={Elsevier}
}

@article{han2020sequentialbatch,
  title={Sequential batch learning in finite-action linear contextual bandits},
  author={Han, Yanjun and Zhou, Zhengqing and Zhou, Zhengyuan and Blanchet, Jose and Glynn, Peter W and Ye, Yinyu},
  journal={arXiv preprint arXiv:2004.06321},
  year={2020}
}

@inproceedings{ruan2021limitedadaptivity,
  title={Linear bandits with limited adaptivity and learning distributional optimal design},
  author={Ruan, Yufei and Yang, Jiaqi and Zhou, Yuan},
  booktitle={Proceedings of the 53rd Annual ACM SIGACT Symposium on Theory of Computing},
  pages={74--87},
  year={2021}
}

@article{ren2024dynamicbatch,
  title={Dynamic batch learning in high-dimensional sparse linear contextual bandits},
  author={Ren, Zhimei and Zhou, Zhengyuan},
  journal={Management Science},
  volume={70},
  number={2},
  pages={1315--1342},
  year={2024},
  publisher={INFORMS}
}

@article{ren2020batchedglm,
  title={Batched learning in generalized linear contextual bandits with general decision sets},
  author={Ren, Zhimei and Zhou, Zhengyuan and Kalagnanam, Jayant R},
  journal={IEEE Control Systems Letters},
  volume={6},
  pages={37--42},
  year={2020},
  publisher={IEEE}
}

@article{sawarni2024limitedadaptivity,
  title={Generalized linear bandits with limited adaptivity},
  author={Sawarni, Ayush and Das, Nirjhar and Barman, Siddharth and Sinha, Gaurav},
  journal={Advances in Neural Information Processing Systems},
  volume={37},
  pages={8329--8369},
  year={2024}
}

@article{hanna2023efficient,
  title={Efficient batched algorithm for contextual linear bandits with large action space via soft elimination},
  author={Hanna, Osama and Yang, Lin and Fragouli, Christina},
  journal={Advances in Neural Information Processing Systems},
  volume={36},
  pages={56772--56783},
  year={2023}
}

@article{ren2024optimalbatchedlinear,
  title={Optimal batched linear bandits},
  author={Ren, Xuanfei and Jin, Tianyuan and Xu, Pan},
  journal={arXiv preprint arXiv:2406.04137},
  year={2024}
}

@inproceedings{assadi2024bestarmevades,
  title={The best arm evades: Near-optimal multi-pass streaming lower bounds for pure exploration in multi-armed bandits},
  author={Assadi, Sepehr and Wang, Chen},
  booktitle={The Thirty Seventh Annual Conference on Learning Theory},
  pages={311--358},
  year={2024},
  organization={PMLR}
}

@article{Raz16space,
  author       = {Ran Raz},
  title        = {Fast Learning Requires Good Memory: {A} Time-Space Lower Bound for
                  Parity Learning},
  journal      = {CoRR},
  volume       = {abs/1602.05161},
  year         = {2016},
  url          = {http://arxiv.org/abs/1602.05161},
  eprinttype   = {arXiv},
  eprint       = {1602.05161},
  bibsource    = {dblp computer science bibliography, https://dblp.org}
}

@article{Raz17space,
  author       = {Sumegha Garg and
                  Ran Raz and
                  Avishay Tal},
  title        = {Extractor-Based Time-Space Lower Bounds for Learning},
  journal      = {CoRR},
  volume       = {abs/1708.02639},
  year         = {2017},
  url          = {http://arxiv.org/abs/1708.02639},
  eprinttype   = {arXiv},
  eprint       = {1708.02639},
  bibsource    = {dblp computer science bibliography, https://dblp.org}
}

@InProceedings{Moshkovitz17space,
  title = {Mixing Implies Lower Bounds for Space Bounded Learning},
  author = {Moshkovitz, Dana and Moshkovitz, Michal},
  booktitle = {Proceedings of the 2017 Conference on Learning Theory},
  pages = {1516--1566},
  year = {2017},
  editor = {Kale, Satyen and Shamir, Ohad},
  volume = {65},
  series = {Proceedings of Machine Learning Research},
  month = {07--10 Jul},
  publisher = {PMLR},
  url = {https://proceedings.mlr.press/v65/moshkovitz17a.html}
}

@article{Moshkovitz17generalspace,
  author       = {Michal Moshkovitz and
                  Naftali Tishby},
  title        = {A General Memory-Bounded Learning Algorithm},
  journal      = {CoRR},
  volume       = {abs/1712.03524},
  year         = {2017},
  url          = {http://arxiv.org/abs/1712.03524},
  eprinttype   = {arXiv},
  eprint       = {1712.03524},
  bibsource    = {dblp computer science bibliography, https://dblp.org}
}

\end{document}